\documentclass[12pt]{article}
\usepackage{amsfonts}
\usepackage{amsmath, amsthm}
\usepackage{epsfig}
\usepackage{rotating}
\usepackage{authblk}
\usepackage[toc,page]{appendix}
\usepackage{mathtools}
\newtheorem{definition}{Definition}[section]
\usepackage{algpseudocode,algorithm,algorithmicx}

\algblock{ParFor}{EndParFor}
\newcommand{\dT}{\top}
\algnewcommand\algorithmicparfor{\textbf{parfor}}
\algnewcommand\algorithmicpardo{\textbf{do}}
\algnewcommand\algorithmicendparfor{\textbf{end\ parfor}}
\algrenewtext{ParFor}[1]{\algorithmicparfor\ #1\ \algorithmicpardo}
\algrenewtext{EndParFor}{\algorithmicendparfor}

\usepackage[T1]{fontenc}

\newtheorem{theorem}{Theorem}

\newtheorem{corollary}{Corollary}
\newtheorem{lemma}{Lemma}
\newcommand*\Let[2]{\State #1 $\gets$ #2}
 
\algrenewcommand\algorithmicrequire{\textbf{Precondition:}}
\algrenewcommand\algorithmicensure{\textbf{Postcondition:}}
\usepackage{amsmath,verbatim,color,amssymb,epsfig}
\usepackage{bm}
\usepackage{amsfonts}
\usepackage{subfig}
\usepackage{epsfig}
\usepackage{multirow}
\usepackage{graphicx}
\usepackage{array}
\usepackage[table]{xcolor}
\definecolor{Gray}{gray}{0.80}
\usepackage{lscape}
\usepackage{cite}
\usepackage{natbib}
\usepackage[normalem]{ulem}
\usepackage[colorlinks=true,urlcolor=blue,citecolor=purple,linkcolor=blue,bookmarks=true]{hyperref}
\DeclareMathOperator*{\arginf}{arginf} % Jan Hlavacek
\DeclareMathOperator*{\argsup}{argsup} % Jan Hlavacek
\usepackage{caption}

\begin{document}

\newtheorem{remark}{Remark}
\makeatletter % make @ act like a letter
\@addtoreset{equation}{section}
\makeatother  % make @ act like a non-letter

\def\yincomment#1{\vskip 2mm\boxit{\vskip 2mm{\color{red}\bf#1} {\color{blue}\bf --Yin\vskip 2mm}}\vskip 2mm}
\def\squarebox#1{\hbox to #1{\hfill\vbox to #1{\vfill}}}
\def\boxit#1{\vbox{\hrule\hbox{\vrule\kern6pt
          \vbox{\kern6pt#1\kern6pt}\kern6pt\vrule}\hrule}}

\def\theequation{\thesection.\arabic{equation}}
\newcommand{\ds}{\displaystyle}

\newcommand{\bJ}{\mbox{\bf J}}
\newcommand{\bF}{\mbox{\bf F}}
\newcommand{\bM}{\mbox{\bf M}}
\newcommand{\bR}{\mbox{\bf R}}
\newcommand{\bZ}{\mboxZ}
\newcommand{\bX}{\mbox{\bf X}}
\newcommand{\bx}{\mbox{\bf x}}
\newcommand{\bQ}{\mbox{\bf Q}}
\newcommand{\bH}{\mbox{\bf H}}
\newcommand{\bh}{\mbox{\bf h}}
\newcommand{\bz}{\mboxZ}
\newcommand{\ba}{\mbox{\bf a}}
\newcommand{\be}{\mbox{\bf e}}
\newcommand{\bG}{\mboxG}
\newcommand{\bB}{\mbox{\bf B}}
\newcommand{\bb}{\mbox{\bf b}}
\newcommand{\bA}{\mbox{\bf A}}
\newcommand{\bC}{\mbox{\bf C}}
\newcommand{\bI}{\mbox{\bf I}}
\newcommand{\bD}{\mbox{\bf D}}
\newcommand{\bU}{\mbox{\bf U}}
\newcommand{\bc}{\mbox{\bf c}}
\newcommand{\bd}{\mbox{\bf d}}
\newcommand{\bs}{\mbox{\bf s}}
\newcommand{\bS}{\mbox{\bf S}}
\newcommand{\bV}{\mbox{\bf V}}
\newcommand{\bv}{\mbox{\bf v}}
\newcommand{\bW}{\mbox{\bf W}}
\newcommand{\bw}{\mbox{\bf w}}
\newcommand{\bg}{\mboxG}
\newcommand{\bu}{\mbox{\bf u}}
\def\bb{{\bf b}}

\newcommand{\bcU}{\boldsymbol{\cal U}}
\newcommand{\bbeta}{\boldsymbol{\beta}}
\newcommand{\bdelta}{\boldsymbol{\delta}}
\newcommand{\bDelta}{\boldsymbol{\Delta}}
\newcommand{\boldeta}{\boldsymbol{\eta}}
\newcommand{\bxi}{\boldsymbol{\xi}}
\newcommand{\bGamma}{\boldsymbol{\Gamma}}
\newcommand{\bSigma}{\boldsymbol{\Sigma}}
\newcommand{\balpha}{\boldsymbol{\alpha}}
\newcommand{\bOmega}{\boldsymbol{ R}}
\newcommand{\btheta}{\boldsymbol{\theta}}
\newcommand{\bmu}{\boldsymbol{\mu}}
\newcommand{\bnu}{\boldsymbol{\nu}}
\newcommand{\bgamma}{\boldsymbol{\gamma}}

\newtheorem{thm}{Theorem}[section]
\newtheorem{lem}{Lemma}[section]
\newtheorem{rem}{Remark}[section]
\newcolumntype{L}[1]{>{\raggedright\let\newline\\\arraybackslash\hspace{0pt}}m{#1}}
\newcolumntype{C}[1]{>{\centering\let\newline\\\arraybackslash\hspace{0pt}}m{#1}}
\newcolumntype{R}[1]{>{\raggedleft\let\newline\\\arraybackslash\hspace{0pt}}m{#1}}

\newcommand{\tabincell}[2]{\begin{tabular}{@{}#1@{}}#2\end{tabular}}

\title{\bf Nonparametric Functional Approximation with Delaunay Triangulation}
\author{Yehong Liu$^1$}
%\thanks{e-mail: liuyh@hku.hk} }
\author{Guosheng Yin$^2$}
\affil{
Department of Statistics and Actuarial Science\\
The University of Hong Kong\\
Pokfulam Road, Hong Kong\\
Email: $^1$liuyh@hku.hk and $^2$gyin@hku.hk}
\renewcommand\Authands{ and }
\maketitle

\begin{abstract}
We propose a differentiable nonparametric algorithm, the Delaunay triangulation learner (DTL), to solve the functional approximation problem on the basis of a $p$-dimensional feature space.
By conducting the Delaunay triangulation algorithm on the data points, the DTL partitions the feature space into a series of $p$-dimensional simplices in a geometrically optimal way, and fits a linear model within each simplex.
We study its theoretical properties by exploring the geometric properties of the Delaunay triangulation, and compare its performance with other statistical learners in numerical studies.

\vspace{0.5cm}
\end{abstract}

\noindent{KEY WORDS:} Convex optimizations, Curvature regularization, Nonparametric regression, Piecewise linear, Prediction.

\section{Introduction}
In recent years, the great success of the deep neural network (DNN) has ignited the machine learning community in formulating predictive models as solutions to differentiable optimization problems. The distributions of samples, e.g., images and languages, are assumed to be concentrated in the regions of some low-dimensional smooth functionals (manifolds), such that even if the samples are not strictly on the manifold, the error can be very small. More specifically, the manifold assumption can be summarized in three folds:
\begin{itemize}
\item [(1)] The underlying model (i.e., the principle functional) for the data distribution is on a low-dimensional functional.
\item [(2)] The principal functional is smooth but complicated in shape.
\item [(3)] The errors between the the functional and real data are very small relative to the variance of the functional.
\end{itemize}
These assumptions are the prerequisites of the success of DNN. First, the DNN is a universal approximator, which ensures its flexibility in shape when approximating the principal surface of the model. Second, the sparse techniques used in DNN, e.g., dropout and pooling, make it possible to regularize the network to concentrate on low-dimensional manifolds. Third, the sample size used in the training process of DNN is usually large, which ensures that DNN can well approximate the principle functional. Therefore, DNN is a successful example of a nonparametric, flexible and differentiable approximator, yet with its only weakness on interpretability.
We focus on the problem of approximating a low-dimensional smooth functional with small errors or without errors, and propose an interpretable, geometrically optimal, nonparametric and differentiable approach, the Delaunay triangulation learner (DTL).

The DTL has a number of attractive properties:
\begin{itemize}
\item [(1)] DTL fits a piecewise linear model, which is well established and can be easily interpreted.
\item [(2)] DTL naturally separates the feature space in a geometrically optimal way, and it has local adaptivity in each separated region when fitting a smooth functional.
\item [(3)] Compared with the general triangulation methods (e.g., random triangulation), where the geometric structure of the triangles can be difficult to analyze, the Delaunay triangulation makes it possible for many stunning geometrical and statistical properties, which lays out a new direction for statistical research on piecewise linear models.
\item [(4)] Based on the construction of the DTL, we can define the regularization function geometrically to penalize the roughness of the fitted function.
\item [(5)] DTL can accommodate multidimensional subspace interactions in a flexible way, which is useful when the output of the model is dependent on the impact of a group of covariates.
\item [(6)] DTL is formulated as a differentiable optimization problem, which can be solved via a series of well-known gradient-based optimizers.
\end{itemize}
To better illustrate the advantages of the DTL both theoretically and empirically, we present the low-dimensional settings where there are only a small number of features ($p\ll n$), while leaving extensions to high-dimensional settings in our future work.

\section{Delaunay Triangulation}\label{Delaunay_Triangulation}
For ease of exposition, we use `triangle' and `simplex' exchangeably throughout this paper, as triangle is a two-dimensional simplex. All definitions and theories are established on general dimension ($p\geq2$).
In geometry, triangulation is often used to determine the location of a point by forming simplex to the point from other known points.
Let $\bold{P}$ be a set of $n$ points in the $p$-dimensional space ($n \geq p+1$), and the convex hull of the points has a nonzero volume.
The Delaunay triangulation of the points is defined as follows.
\begin{definition}
Given a set of points $\bold{P}$ in the $p$-Euclidean space,
the Delaunay triangulation $DT(\bold{P})$ is a triangulation such that no point in $\bold{P}$ is inside the circumscribed spheres of other triangles in $DT(\bold{P})$.
\end{definition}
This definition is illustrated in Figure \ref{fig:Delaunay_plots} (a).
Geometrically, it has been well established that for any given points $\bold{P}$ in general position (the convex hull of $\bold{P}$ is non-degenerate in $p$-dimensional space), there exists
only one Delaunay triangulation $DT(\bold{P})$, i.e., the Delaunay triangulation is unique \citep{Delaunay1934}.
Intuitively, the Delaunay triangulation algorithm results in a group of simplices that are most regularized in shape, in comparison to any other type of triangulation.
As shown in Figure \ref{fig:Delaunay_plots} (b) and (c), in a two-dimensional space, the mesh constructed by the Delaunay triangulation has fewer skinny triangles compared with that by a random triangulation.
Figure \ref{fig:Delaunay_Triangulation_surface} shows a saddle surface and its Delaunay triangulation.
\begin{figure}[H]
	\subfloat[]{%
    \includegraphics[width=0.3\linewidth]{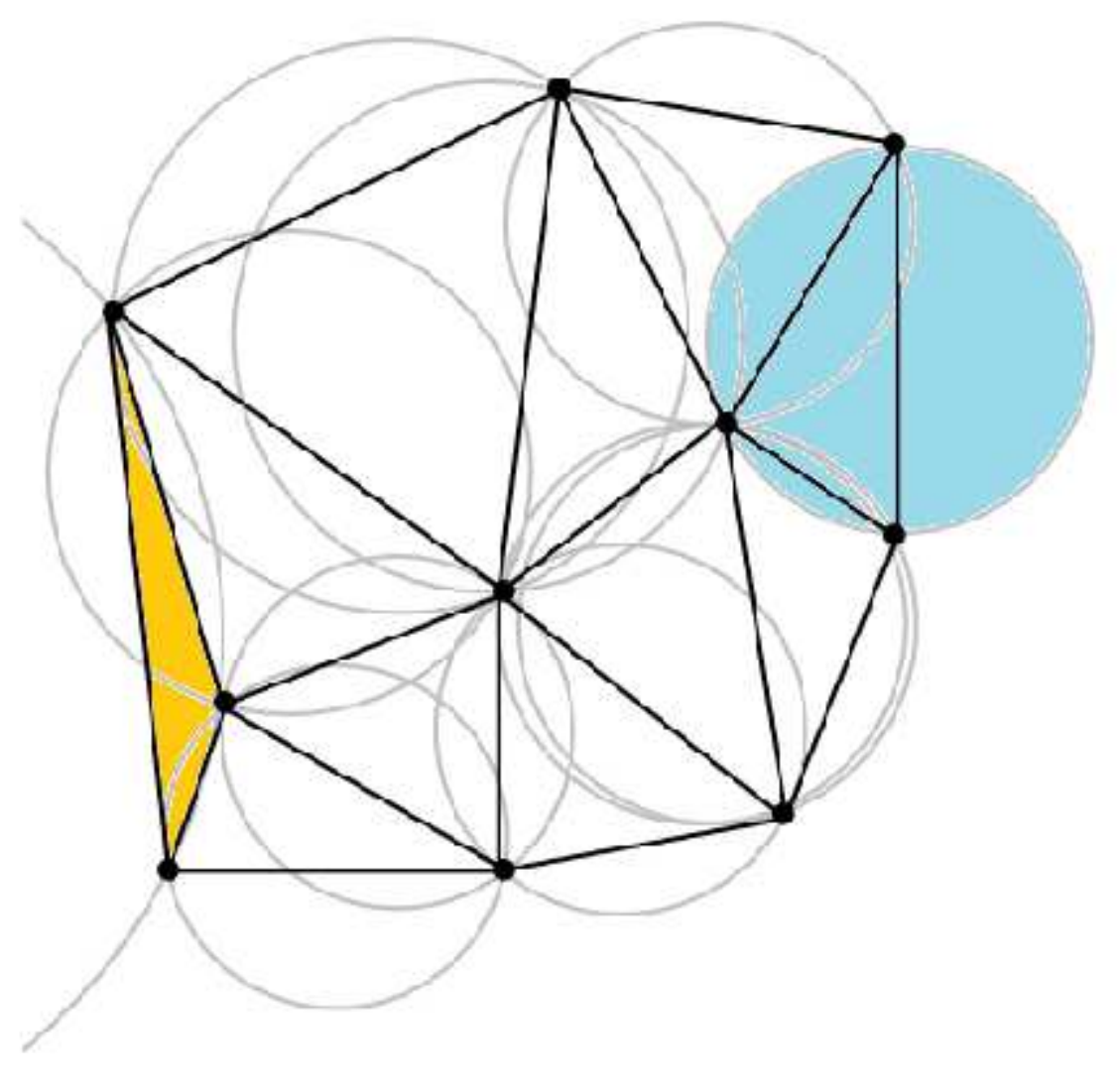}
    }
    \hfill
    \subfloat[]{%
    \includegraphics[width=0.3\linewidth]{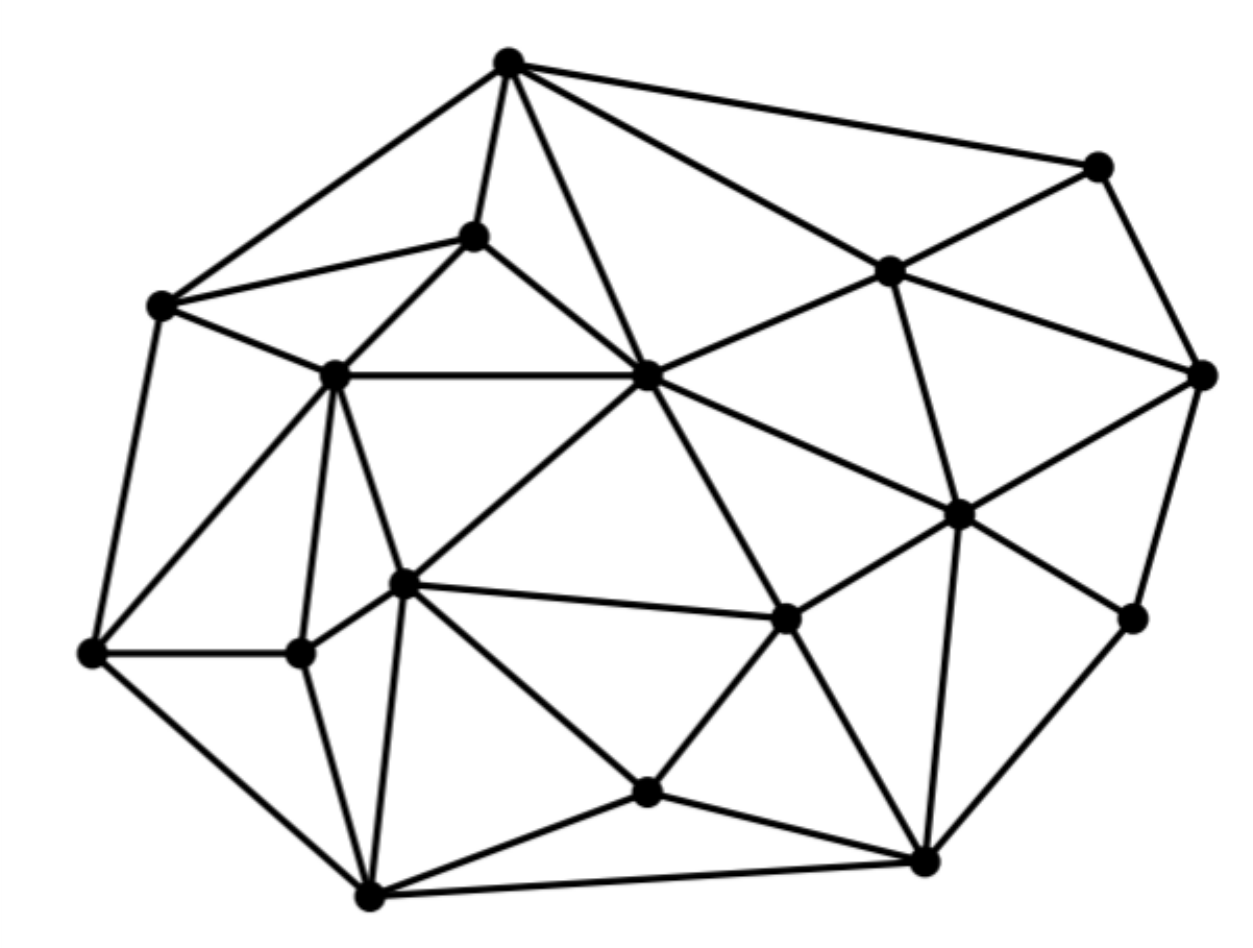}
    }
    \hfill
    \subfloat[]{%
    \includegraphics[width=0.3\linewidth]{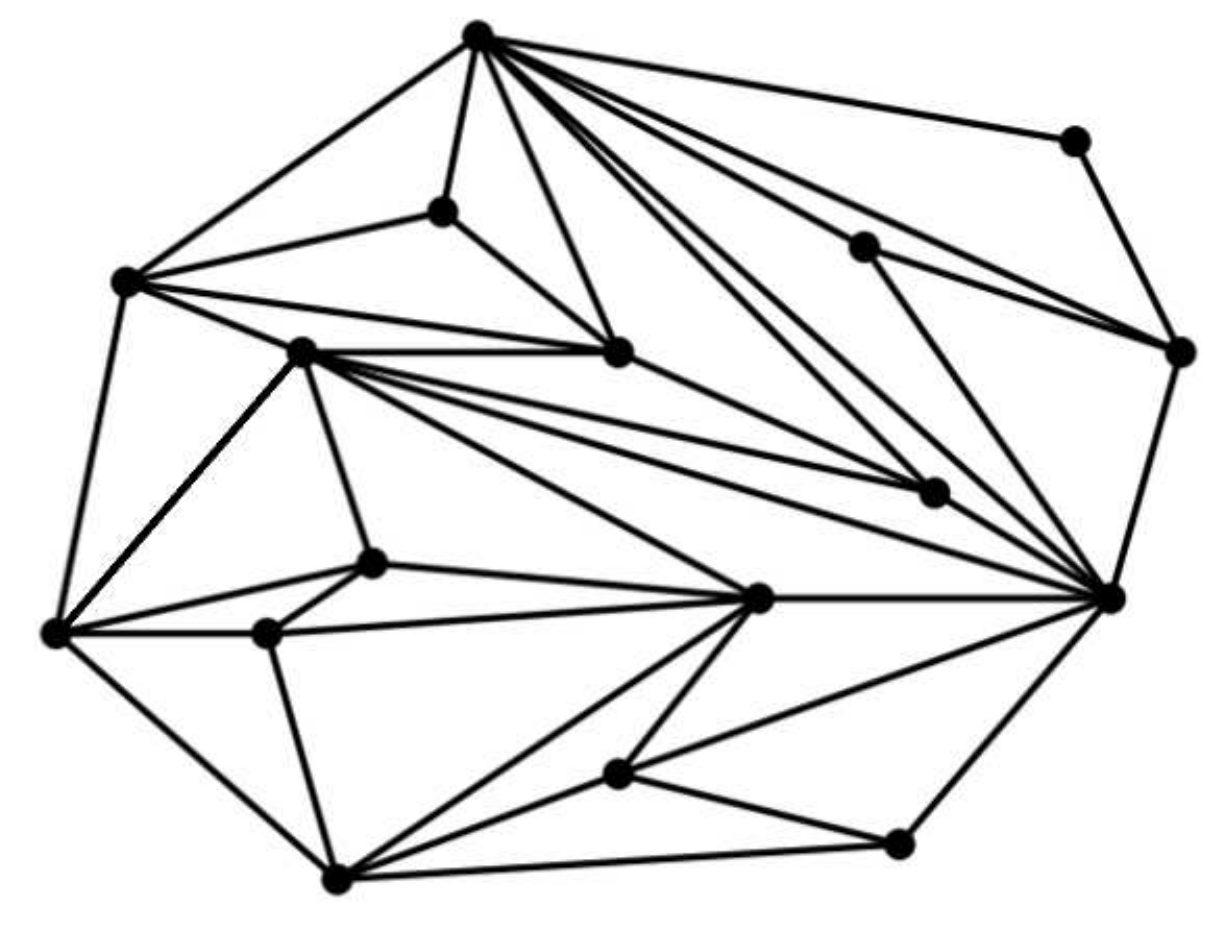}
    }
    \hfill
    \caption{(a) Illustration of the Delaunay triangulation: The circumcircle of any triangle in $DT(\bold{P})$ does not include any other point in $\bold{P}$, e.g., the circumcircle in green does not contain any other point. In a two-dimensional space, the Delaunay triangulation maximizes the smallest angle in all the triangles, e.g., the smallest angle of the $DT(\bold{P})$ is the smallest one in the yellow triangle; (b) Delaunay triangulation; (c) random triangulation.}
    \label{fig:Delaunay_plots}
\end{figure}

Several algorithms can be used to implement the Delaunay triangulation, including the flip algorithm, the Bowyer--Watson algorithm \citep{Bowyer1981}, and \citep{Watson1981}, and the divide-and-conquer paradigm \citep{Cignoni1998}.
The flip algorithm is fairly straightforward as it constructs the triangulation of points and flips the edges until every triangle is of the Delaunay type. It takes $O(n^2)$ time for the edge-flipping operation, where $n$ is the number of points.
The Bowyer--Watson algorithm is relatively more efficient as it takes $O(n \log(n))$ time, by repeatedly adding one vertex at a time and re-triangulating the affected parts of the graph.
In the divide-and-conquer algorithm, a line is drawn recursively to split the vertices into two groups, and the Delaunay triangulation is then implemented for each group. The computational time can be reduced to $O(n \log \log n)$ owing to some tuning techniques.
For higher dimensional cases ($p>2$),
\cite{Fortune1992} proved that the plane-sweep approach produces the worst case time complexity of $O(n^{[p/2]})$ \citep{Klee2015}.
In this paper, the Delaunay triangulation algorithm is implemented by using the Python SciPy package \citep{Jones2001}.

\begin{figure}[t]
    \subfloat[]{%
    \includegraphics[width=0.5\linewidth]{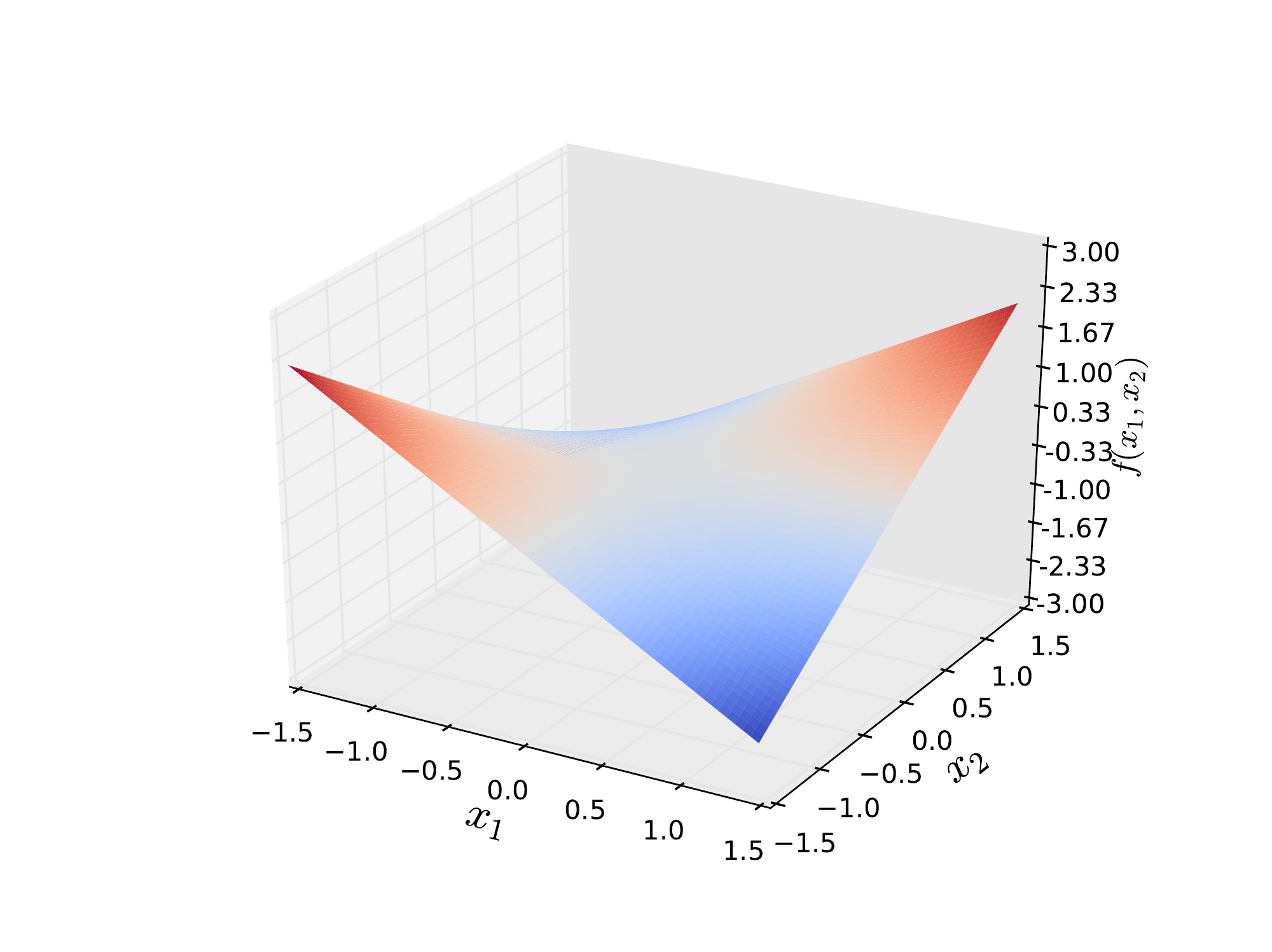}
    }
    \hfill
    \subfloat[]{%
    \includegraphics[width=0.5\linewidth]{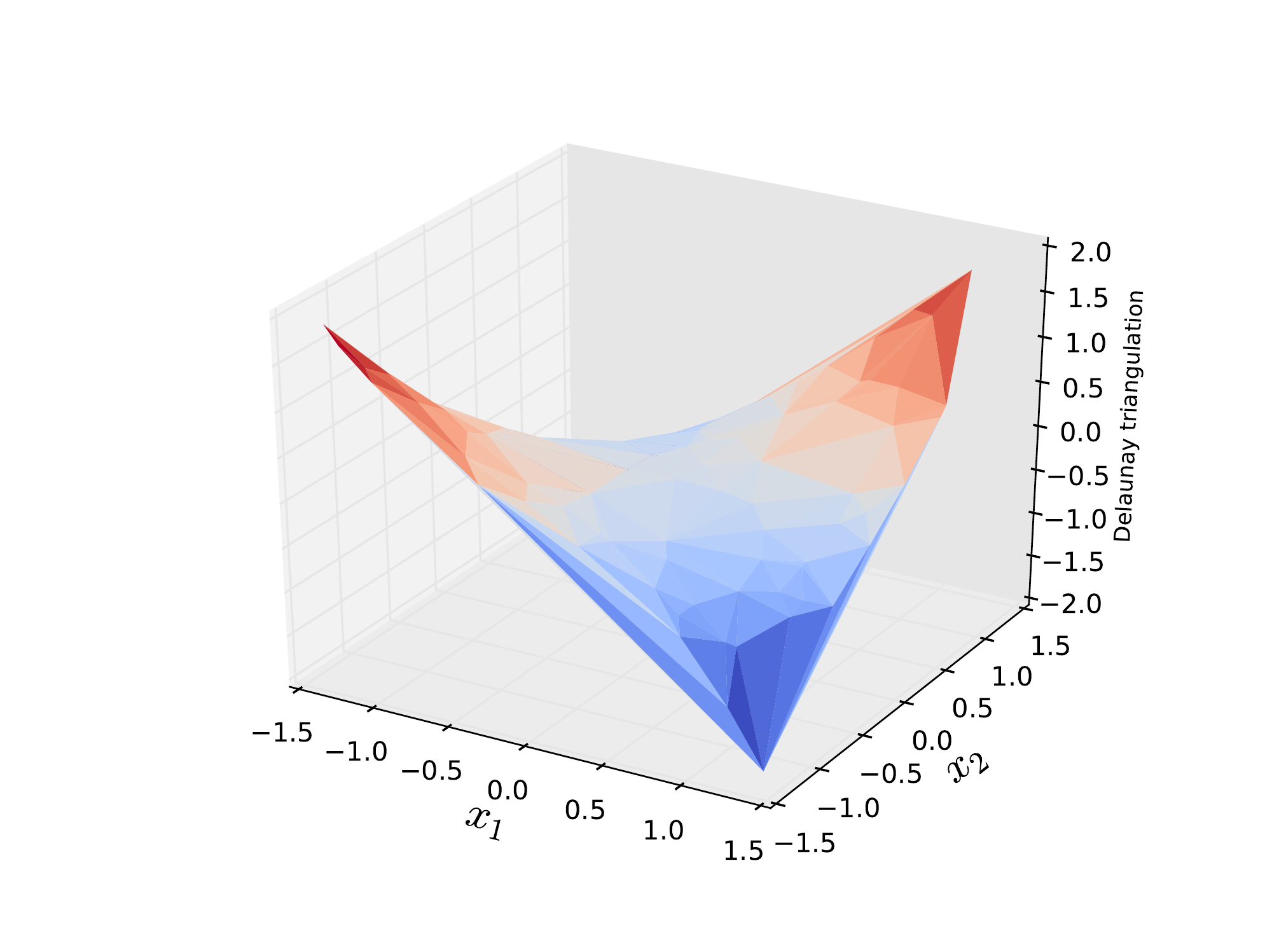}
    }
    \hfill
    \caption{(a) The true function $f(x_1, x_2)=x_1 x_2$, (b) the Delaunay triangulation of the surface based on 100 random samples.}
    \label{fig:Delaunay_Triangulation_surface}
\end{figure}

\section{Delaunay Triangulation Learner}\label{DTL}
\subsection{Formalization of DTL}
As shown in Figure \ref{fig:Delaunay_Triangulation_surface}, given the data $\left\{\mathcal{X}, \mathcal{Y}\right\}$ from a probabilistic model (joint distribution) $\mathcal{P_{X, Y}}$, where $\mathcal{Y}=\left\{y_i \right\}_{i=1}^n$ and $\mathcal{X}=\left\{\bold{x}_i \right\}_{i=1}^n$, the DTL is formalized as follows.

\begin{itemize}

\item[(1)] Delaunay Partition\\
We partition the feature space with a triangle mesh by conducting Delaunay triangulation on $\mathcal{X}=\left\{\bold{x}_i \right\}_{i=1}^n$ and obtain a set of simplices, denoted as $\mathcal{D}\left(\mathcal{X}\right)$.
The convex hull of the Delaunay triangulation, denoted as $\mathcal{H}\left(\mathcal{X}\right)$, is unique by definition.
%a point $\bold{x} \in \mathcal{H}$ if the point $\bold{x}$ is inside the convex hull $\mathcal{H}$.
For any point $\bold{x}$ in the feature space, if $\bold{x}\in \mathcal{H}\left(\mathcal{X}\right)$, there exists a simplex in $\mathcal{D}\left(\mathcal{X}\right)$, denoted as $\mathcal{S}(\bold{x}, \mathcal{X})$, that contains $\bold{x}$,
\begin{eqnarray*}
\mathcal{S}(\bold{x}, \mathcal{X})=
\argsup_{\mathcal{S}\in \mathcal{D}(\mathcal{X})} I\left(\bold{x}\in \mathcal{S}\right), & \text{if  $\bold{x} \in \mathcal{H}(\mathcal{X})$},
\end{eqnarray*}
where $I\left(\bold{x}\in \mathcal{S}\right)$ is the indicator function.
If $\bold{x}\notin \mathcal{H}(\mathcal{X})$, there exists a point $\bold{x}(\mathcal{X})$ nearest to $\bold{x}$, with the corresponding response $y(\bold{x}(\mathcal{X})).$

\item[(2)] Parabolic Lifting\\
Let $\bold{\Psi}=\left(\psi_1,\ldots, \psi_n \right)^\dT$ denote a vector of location parameters.
We construct the DTL as a linear interpolation function based on $\left\{\mathcal{X}, \bold{\Psi}\right\}$, which is denoted as $F_{D}(\bold{x};  \bold{\Psi})$.
The linear interpolation function takes the form
\[
    F_{D}(\bold{x}; \bold{\Psi})=
\begin{cases}
    g(\bold{x}; \mathcal{S}(\bold{x}, \mathcal{X}), \bold{\Psi}), & \text{if  $\bold{x} \in \mathcal{H}(\mathcal{X})$},\\
   y(\bold{x}(\mathcal{X})), & \text{if $\bold{x}\not\in \mathcal{H}(\mathcal{X})$},
\end{cases}
\]
where $g(\bold{x}; \mathcal{S}(\bold{x}, \mathcal{X}), \bold{\Psi})$ is a $p$-dimensional linear function defined on the simplex $\mathcal{S}(\bold{x}, \mathcal{X})$, satisfying
\begin{eqnarray*}
g(\bold{x}_i; \mathcal{S}(\bold{x}_i, \mathcal{X}), \bold{\Psi})=\psi_{i}, \quad i=1, \ldots, n.
\end{eqnarray*}
\end{itemize}

In summary, the DTL partitions the feature space with a Delaunay triangulation algorithm. It is a nonparametric functional learner as the dimension of the parameter $\bold{\Psi}$ grows at the same pace as the sample size $n$.

\subsection{Optimization}
In the DTL functional optimization problem, one has a response variable $y$ and a vector of input variables $\bold{x}=\left(x_1,\ldots,x_p\right)^\dT$ in a $p$-dimensional feature space. Given the data $\left\{\mathcal{X}, \mathcal{Y}\right\}$, our goal is to obtain an estimator $\hat{F}_D$ by solving an optimization problem,
\begin{eqnarray*}
\bold{\Psi}^* = \arginf_\bold{\Psi} L(y, F_{D}(\bold{x}; \bold{\Psi})) + \lambda R(F_{D}(\bold{x}; \bold{\Psi})),
\end{eqnarray*}
where $L(y, F_{D}(\bold{x}; \bold{\Psi}))$ is the loss function, $\lambda$ is a tuning parameter and $R(F_{D}(\bold{x}; \bold{\Psi}))$ is a regularization function measuring the roughness of $F_{D}(\bold{x}; \bold{\Psi})$.
The loss function can be squared loss or absolute loss for regression problems, and exponential loss for classification problems.

As for the regularization function $R(\cdot)$, we propose using the total discrete curvature as a measure of roughness.
Suppose that a vertex ${\bf v}$ is inside the convex hull $\mathcal{H}(\mathcal{X})$, and let $\mathcal{N}({\bf v})$ denote the set of simplices that contain ${\bf v}$.
As a linear interpolation function, the DTL is strictly linear in each triangulated region. Through the function $F_D(\bold{x}; \bold{\Psi})$, each simplex $\mathcal{S}\in \mathcal{N}({\bf v})$ has a corresponding $p$-dimensional simplex $\mathcal{S}^F$ as its functional image. More specifically,
$\mathcal{S}^F= \{(\bold{x}, F_D(\mathcal{\bold{x}}, \bold{\Psi}))|\bold{x}\in\mathcal{S}\}$
is a $p$-dimensional simplex for any $\mathcal{S}$. Furthermore, we define a $(p+1)$-dimensional vector $\bold{n}(\mathcal{S})$ as the standard up-norm vector of $\mathcal{S}^F$, which satisfies
\begin{equation}\label{eqn:up_norm}
\begin{cases}
&\langle \bold{n}(\mathcal{S}), (\bold{G}^{\dT}, \bold{G}^{\dT}\bold{G})^{\dT}\rangle = 0,\\
&\langle \bold{n}(\mathcal{S})[1:p], \bold{G}^{\dT}\rangle = 0,\\
&\|\bold{n}(\mathcal{S})\|=1,\\
&\langle \bold{n}(\mathcal{S}), {\bf e}_{p+1}\rangle \geq 0,
\end{cases}
\end{equation}
where ${\bf e}_{p+1}=(\bold{0}_p, 1)^{\dT}$, $\bold{n}(\mathcal{S})[1:p]$ is the vector composed of the first $p$-dimensional elements of $\bold{n}(\mathcal{S})$, and $\bold{G}=\bold{\Delta}_{\bold{x}_i}^{-1}\Delta_{\psi_i}$ is a $p$-dimensional gradient vector of $\mathcal{S}^F$,
where
\begin{eqnarray*}
\bold{\Delta}_{\bold{x}_i}&=& \left(\bold{x}_i - \bold{x}_1,\ldots, \bold{x}_i - \bold{x}_{i-1}, \bold{x}_i - \bold{x}_{i+1}, \ldots, \bold{x}_i - \bold{x}_{p+1}\right)^{\dT},\\
\Delta_{\psi_i}&=& \left(\psi_i - \psi_1,\ldots, \psi_i - \psi_{i-1}, \psi_i - \psi_{i+1},\ldots, \psi_i - \psi_{p+1}\right)^{\dT}.
\end{eqnarray*}
Here, $\bold{x}_1, \ldots, \bold{x}_{p+1}$ are the vertices of the same simplex and $\psi_1, \ldots, \psi_{p+1}$ are the corresponding location parameters of the DTL.
The first equation in (\ref{eqn:up_norm}) ensures that the vector $\bold{n}(\mathcal{S})$ is orthogonal to the simplex $\mathcal{S}^F$. The second implies that the vector is normalized, and the third requires the vector to be in the upward direction.

\begin{figure}[t]
	\center
    \includegraphics[scale=0.35]{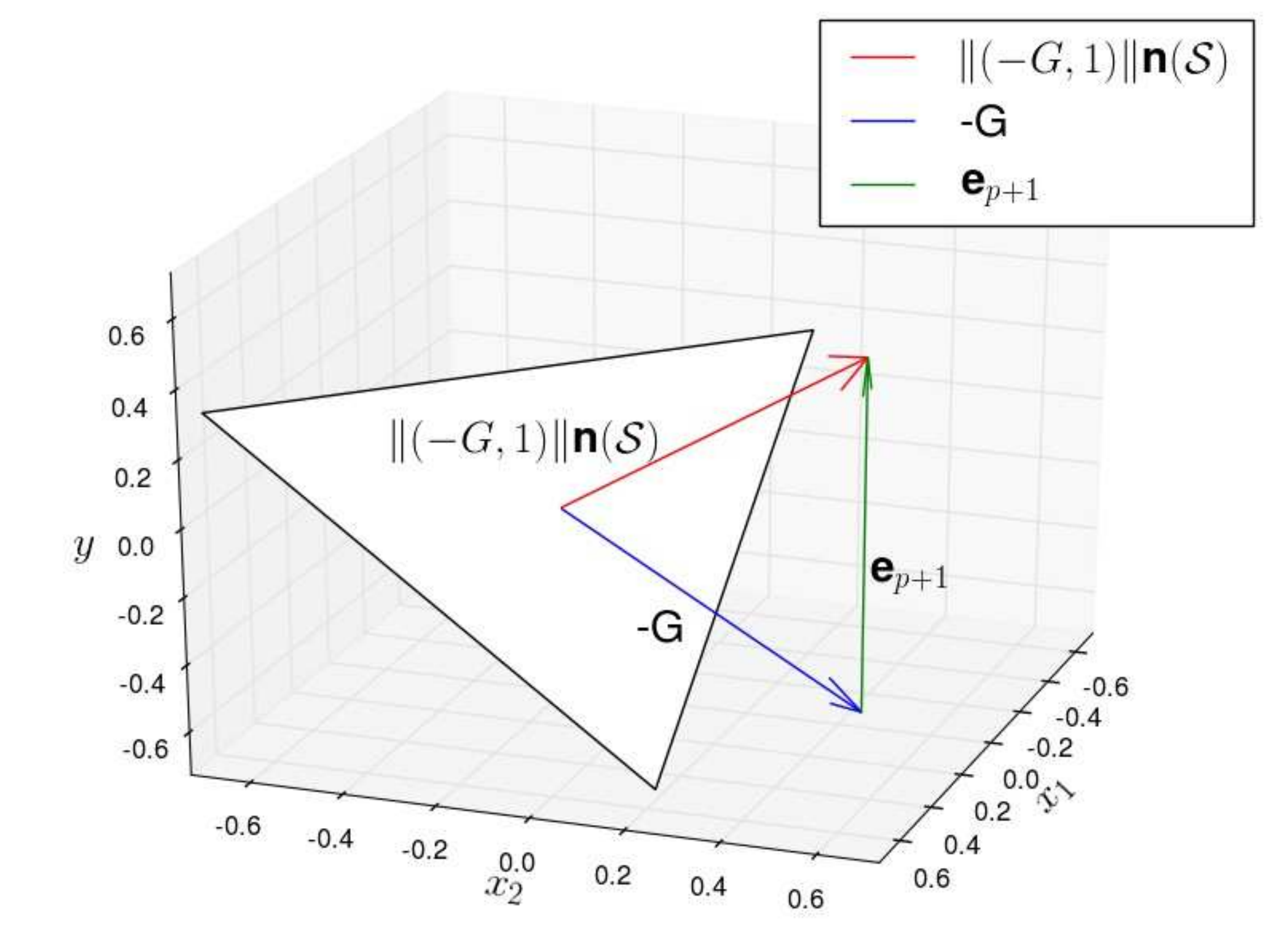}
    \caption{The up-norm vector $\bold{n}(\mathcal{S})$, the negative gradient vector $-\bold{G}$, and the vector of axis $\bold{e}_{p+1}$ in (\ref{eqn:up_norm}).}
    \label{fig:interpret_G}
\end{figure}
\begin{figure}[htbp]
	\center
    \includegraphics[scale=0.5]{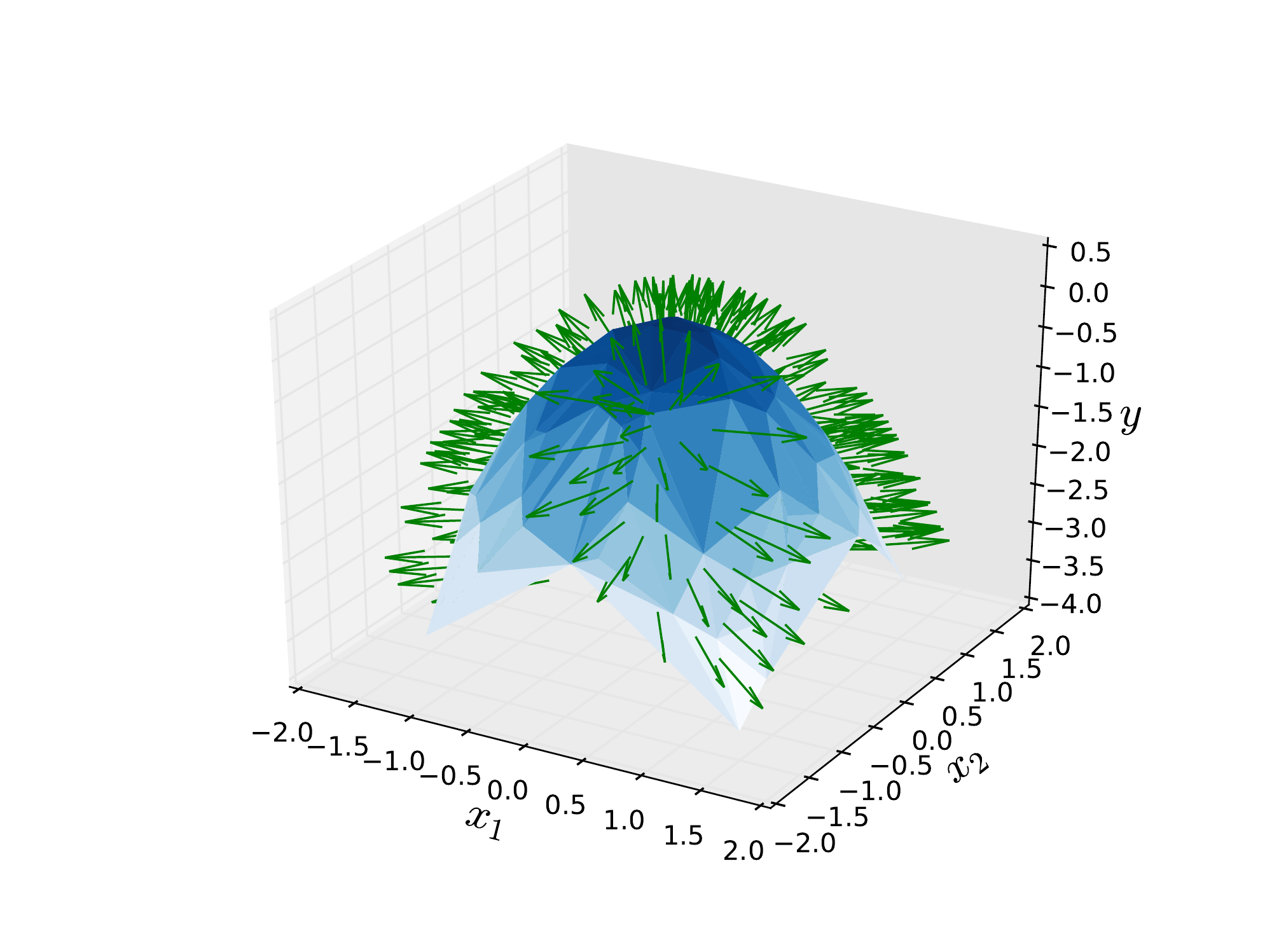}
    \caption{The up-norm vectors of a triangulated surface, where the outward arrows represent the up-norm vectors for each simplex.}
    \label{fig:Normvectors}
\end{figure}
As the up-norm vector shares the same projection as the gradient vector on the feature space, it takes the form of
$\bold{n}(\mathcal{S})=\frac{(-\bold{G}^{\dT}, c)^{\dT}}{\|(-\bold{G}^{\dT}, c)^{\dT}\|}$,
where $c$ is an unknown constant.
By solving the equations in (\ref{eqn:up_norm}), we obtain $c=1$, and thus the standard up-norm vector $\bold{n}(\mathcal{S})$ can be written as
\begin{eqnarray}\label{up-norm}
\bold{n}(\mathcal{S})=\frac{(-\bold{G}^{\dT}, 1)^{\dT}}{\|(-\bold{G}^{\dT}, 1)^{\dT}\|}.
\end{eqnarray}
Figure \ref{fig:interpret_G} exhibits the geometric view of $\bold{n}(\mathcal{S})$, and Figure \ref{fig:Normvectors} shows a plot of the standard up-norm vectors on a discrete surface.

Let $|\mathcal{N}({\bf v})|$ denote the number of elements in $\mathcal{N}({\bf v})$, and define the total discrete curvature as
\begin{eqnarray*}
K({\bf v}) = {
{|\mathcal{N}({\bf v})| \choose 2}}^{-1}{\sum_{\substack{ k < j \\ \mathcal{S}_k, \mathcal{S}_j \in \mathcal{N}({\bf v})}} \measuredangle \left(\bold{n}(\mathcal{S}_k), \bold{n}(\mathcal{S}_j)\right)},
\end{eqnarray*}
where $\measuredangle \left( \bold{n}(\mathcal{S}_k), \bold{n}(\mathcal{S}_j) \right)$ represents the degree of the angle between $\bold{n}(\mathcal{S}_k)$ and $\bold{n}(\mathcal{S}_j)$, and $|\mathcal{N}({\bf v})| \choose 2$ is the number of all possible combinations for $\left(\bold{n}(\mathcal{S}_k), \bold{n}(\mathcal{S}_j)\right)$.
The total discrete curvature is calculated by measuring the angle between each pair of the up-norm vectors and averaging the total degree of the angles grouped by each vertex.
A special case is that all the simplices of the DTL collectively form a hyperplane in the $(p+1)$-dimension, where all the standard up-norm vectors are identical and the total discrete curvature equals zero.
As the cosine of an angle formed by two unit vectors is equal to the inner product of the vectors, it is equivalent to defining the regularization function as
\begin{eqnarray}\label{eqn:regularization}
 R\left(F_{D}(\bold{x};\bold{\Psi})\right)
&=& \sum_{i=1}^n \left\{1 - {|\mathcal{N}({\bf v})| \choose 2}^{-1}\sum_{\substack{ k < j \\ \mathcal{S}_k, \mathcal{S}_j \in \mathcal{N}(\mathcal{\bold{x}}_i)}} \big\langle \bold{n}(\mathcal{S}_k), \bold{n}(\mathcal{S}_j)  \big\rangle
  \right\}.
\end{eqnarray}

Both the loss function and the regularization function are differentiable with respect to $\bold{\Psi}$, thus all the gradient-based optimizers can be applied to solve the optimization problem of the DTL.
The gradient of the objective function
$L\left(y, F_{D}(\bold{x}; \bold{\Psi})\right) + \lambda R\left(F_{D}(\bold{x}; \bold{\Psi})\right)$
is given by
$\nabla = \sum_{i=1}^n \frac{\partial L}{\partial \psi_i} +
\lambda \sum_{i=1}^n \frac{\partial  R}{\partial \psi_i}.$
Algorithm \ref{alg:Adam} displays the Adaptive Moment Estimation (Adam) algorithm \citep{Kingma2015} that iteratively updates the vector $\bold{\Psi}$, where $\beta_1=0.9$, $\beta_2=0.999$ and $\varepsilon=10^{-8}$.
\begin{algorithm}[h]
  \caption{Adam Algorithm for the DTL Optimization
    \label{alg:Adam}}
  \begin{algorithmic}[1]
    \Statex
      \Let{$t$}{$0$} \Let{$\bold{\Psi}^{(t)}$}{$\mathcal{Y}$, }\Let{$\bold{\Psi}^{(t+1)}$}{$\mathcal{Y}+\eta\bold{1}$}, where $\bold{1}=(1, \ldots, 1)^\dT,$
      \While{$\|\bold{\Psi}^{(t+1)}-\bold{\Psi}^{(t)}\|>\eta$ ($\eta$ is the convergence criterion)},
      \Let{$t$}{$t+1$}
        \Let{$\nabla_t$} { $\frac{\partial L}{\partial \bold{\Psi}}\big|_{\bold{\Psi}^{(t)}} +
\lambda \frac{\partial  R}{\partial \bold{\Psi}} \big|_{\bold{\Psi}^{(t)}}$}
        \Let {$m_t$}{$\beta_1^t m_{t-1} + (1-\beta_1^t)\nabla_t$}
        \Let {$v_t$}{$\beta_2^t v_{t-1} + (1-\beta_2^t)\nabla_t$}
        \Let {$\hat{m}_t$}{$\frac{m_t}{1-\beta_1^t}$}
        \Let {$\hat{v}_t$}{$\frac{v_t}{1-\beta_2^t}$}

		\Let{$\bold{\Psi}^{t+1}$}{$\bold{\Psi}^{t} - \frac{\varepsilon}{\sqrt{\hat{v}_t}+\varepsilon}\hat{m}_t \nabla_t$}, $\varepsilon$ is a small step size.

      \EndWhile
      \State \Return{$\bold{\Psi}^{(t+1)}$}
  \end{algorithmic}
\end{algorithm}

\subsection{DTL Advantages}
Figure \ref{fig:Delaunay_Triangulation_learner} shows a comparison between the DTL, decision tree, and multivariate adaptive regression splines (MARS), fitted on the same data generated from a saddle surface model, $y=\arctan(x_1+x_2)+ \epsilon$, where the noise $\epsilon$ follows the normal distribution $N(0, 0.01)$.
In comparison with the decision tree and MARS, the DTL appears to be smoother and more flexible in shape.
The tree-based method tends to be rough when estimating smooth functionals, as
it is not easy to smooth out by fitting more complicated models in the terminal nodes.
The MARS show better adaptivity in handling smooth nonlinear functions than decision trees.
The DTL provides an intrinsic way to partition the feature space with the Delaunay triangulation algorithm, and it has the advantage of balancing the
roughness and smoothness depending on the local availability of data and that can lead to better predictive accuracy.
Figure \ref{fig:Shrinkage_Behaviour} shows the behavior of the regularized DTL with different values of $\lambda$ when fitting a two-dimensional linear model under a squared loss function.

\begin{figure}[htbp]
    \subfloat[Original data]{%
    \includegraphics[width=0.5\linewidth]{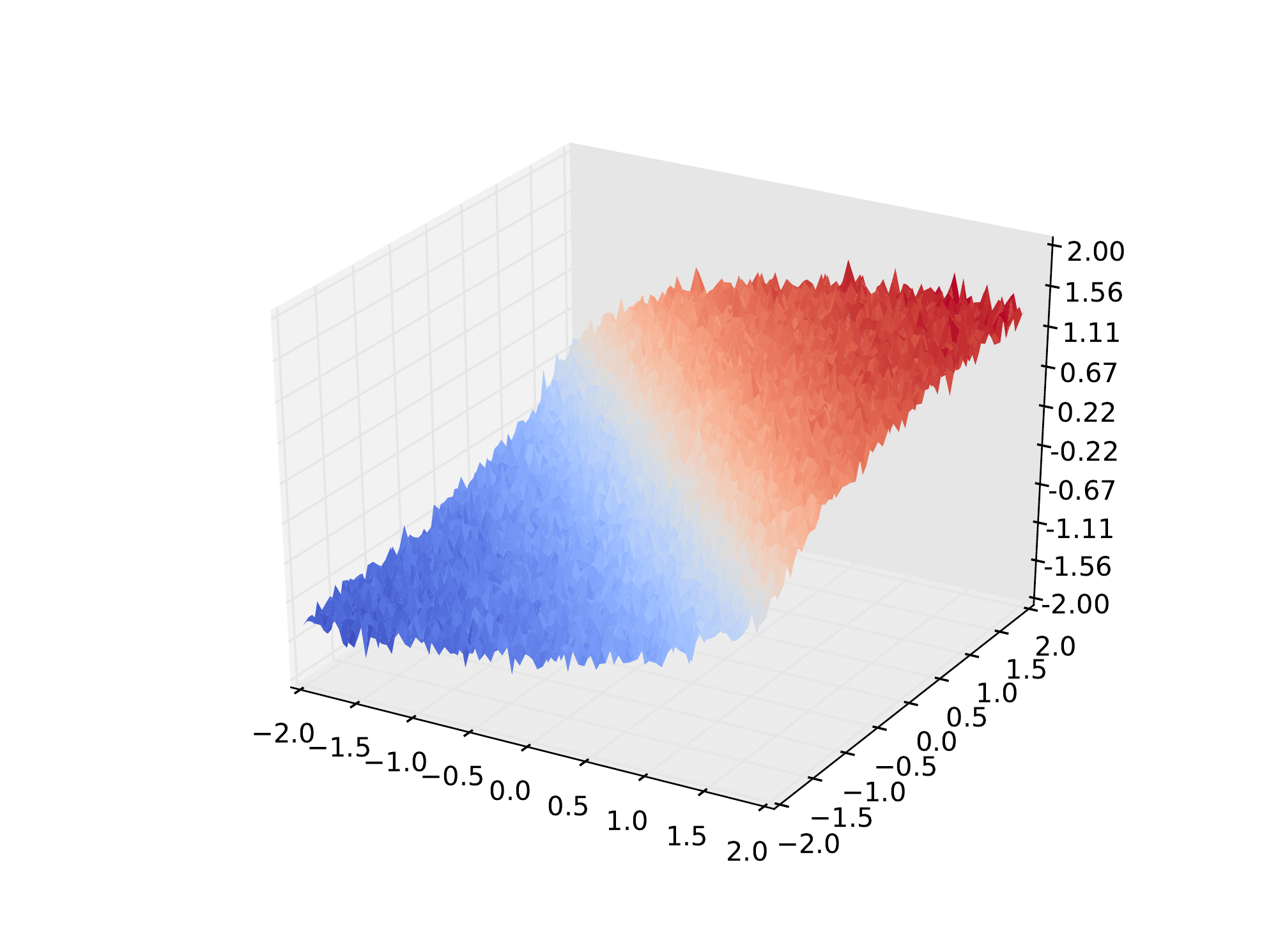}
    }
    \hfill
    \subfloat[Delaunay triangulation learner]{%
    \includegraphics[width=0.5\linewidth]{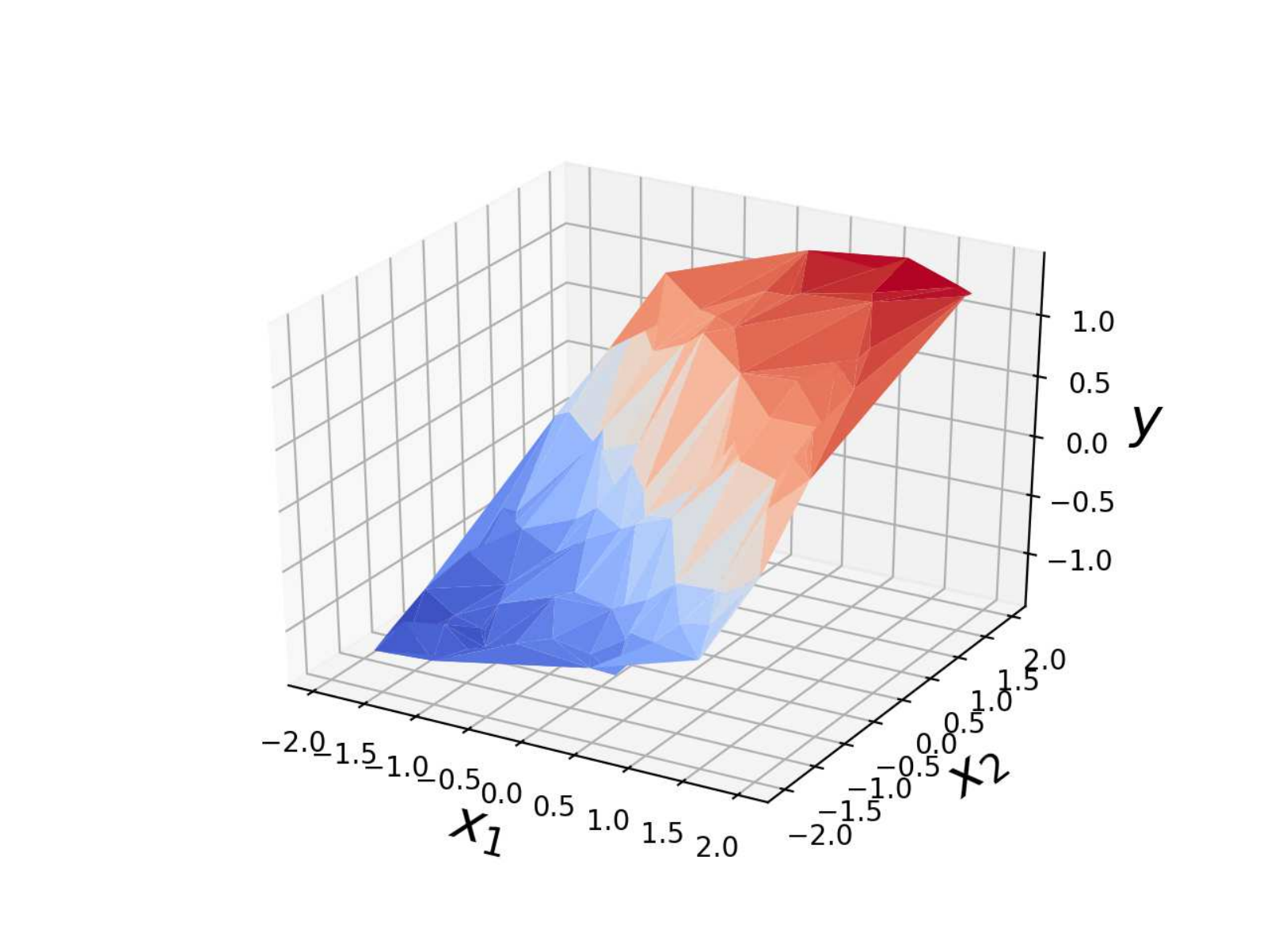}
    }
    \hfill
    \subfloat[Tree-based learner]{%
    \includegraphics[width=0.5\linewidth]{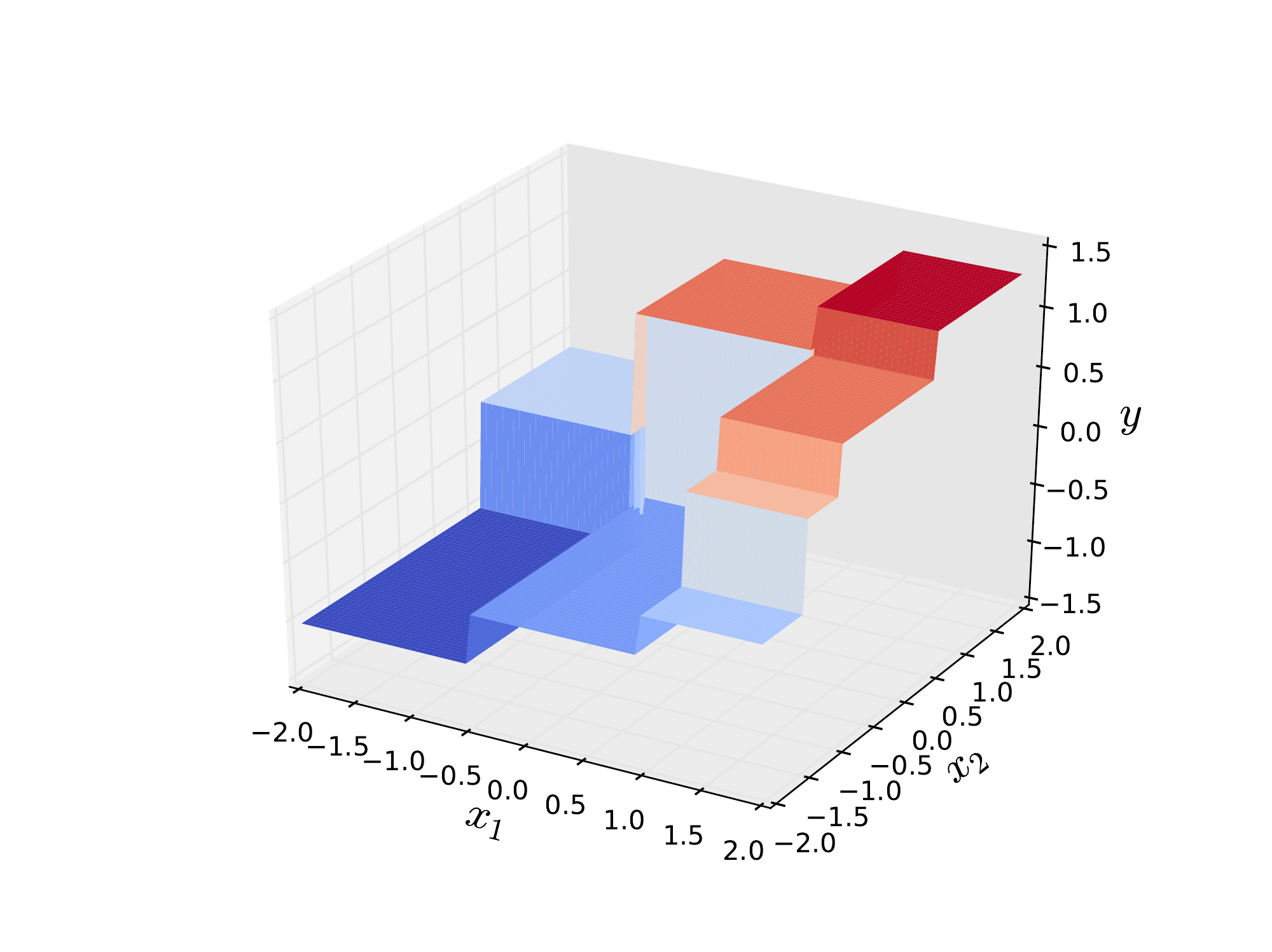}
    }
    \hfill
    \subfloat[Multivariate adaptive regression splines]{%
    \includegraphics[width=0.5\linewidth]{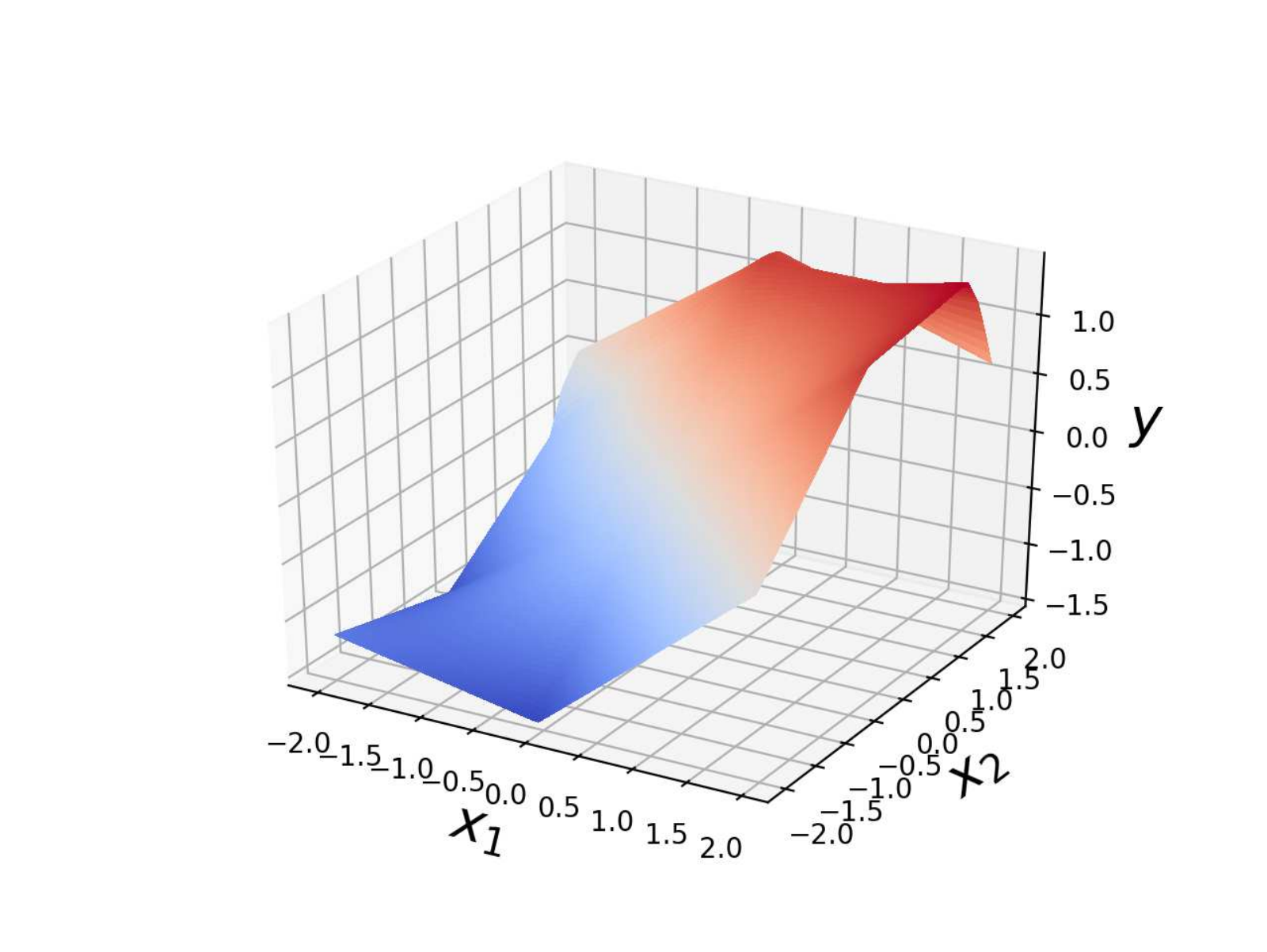}
    }
    \hfill
    \caption{(a) Samples generated from an arc-tangent model, (b) the Delaunay triangulation learner, (c) the tree-based learner, and (d) the Multivariate adaptive regression spline. }
    \label{fig:Delaunay_Triangulation_learner}
\end{figure}

\begin{comment}
\begin{figure}[htbp!]
    \subfloat[$\lambda=0$]{%
    \includegraphics[width=0.5\linewidth]{Compare_Delaunay_with_Tree_a.eps}
    }
    \hfill
    \subfloat[$\lambda=1$]{%
    \includegraphics[width=0.4\linewidth]{Compare_Delaunay_with_Tree_b.eps}
    }
    \hfill
    \subfloat[$\lambda=2$]{%
    \includegraphics[width=0.5\linewidth]{Compare_Delaunay_with_Tree_c.eps}
    }
    \hfill
    \subfloat[$\lambda=10$]{%
    \includegraphics[width=0.5\linewidth]{mars.eps}
    }
    \hfill
    \caption{Comparison of the smoothness of the regularized DTL using different values of $\lambda$, with the data from the linear model $y=X_1 + X_2 + \epsilon$. }
    \label{fig:Shrinkage_Behaviour}
\end{figure}
\end{comment}

\section{Theoretical Properties}\label{Geometrical and Statistical Properties}
We first study the geometrical optimality and asymptotic properties of the Delaunay triangulation mesh under general random distributions, and then these properties are used as cornerstones to find the statistical properties of the DTL.
\begin{figure}[htbp!]
    \subfloat[$\lambda=0$]{%
    \includegraphics[width=0.5\linewidth]{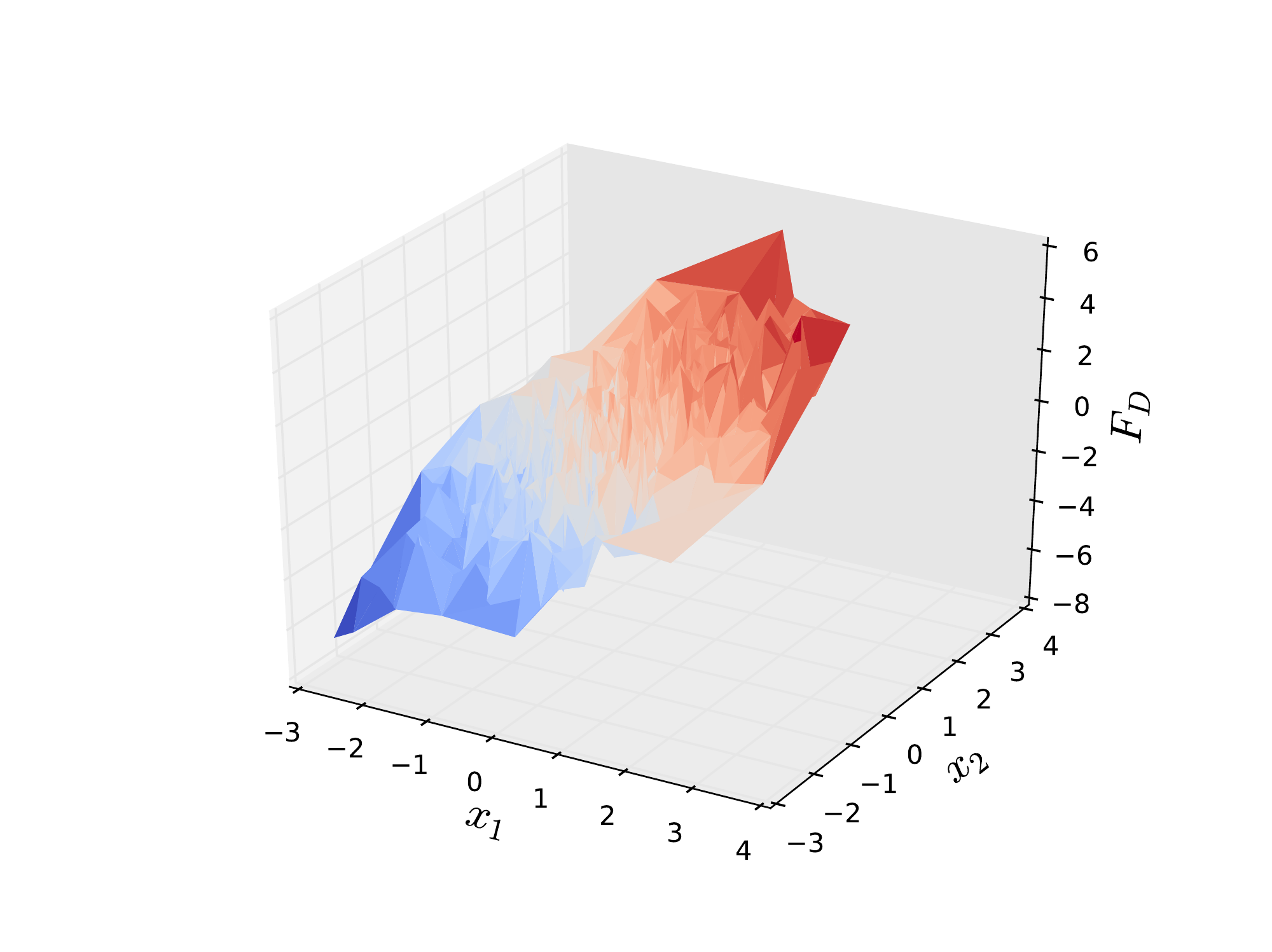}
    }
    \hfill
    \subfloat[$\lambda=1$]{%
    \includegraphics[width=0.5\linewidth]{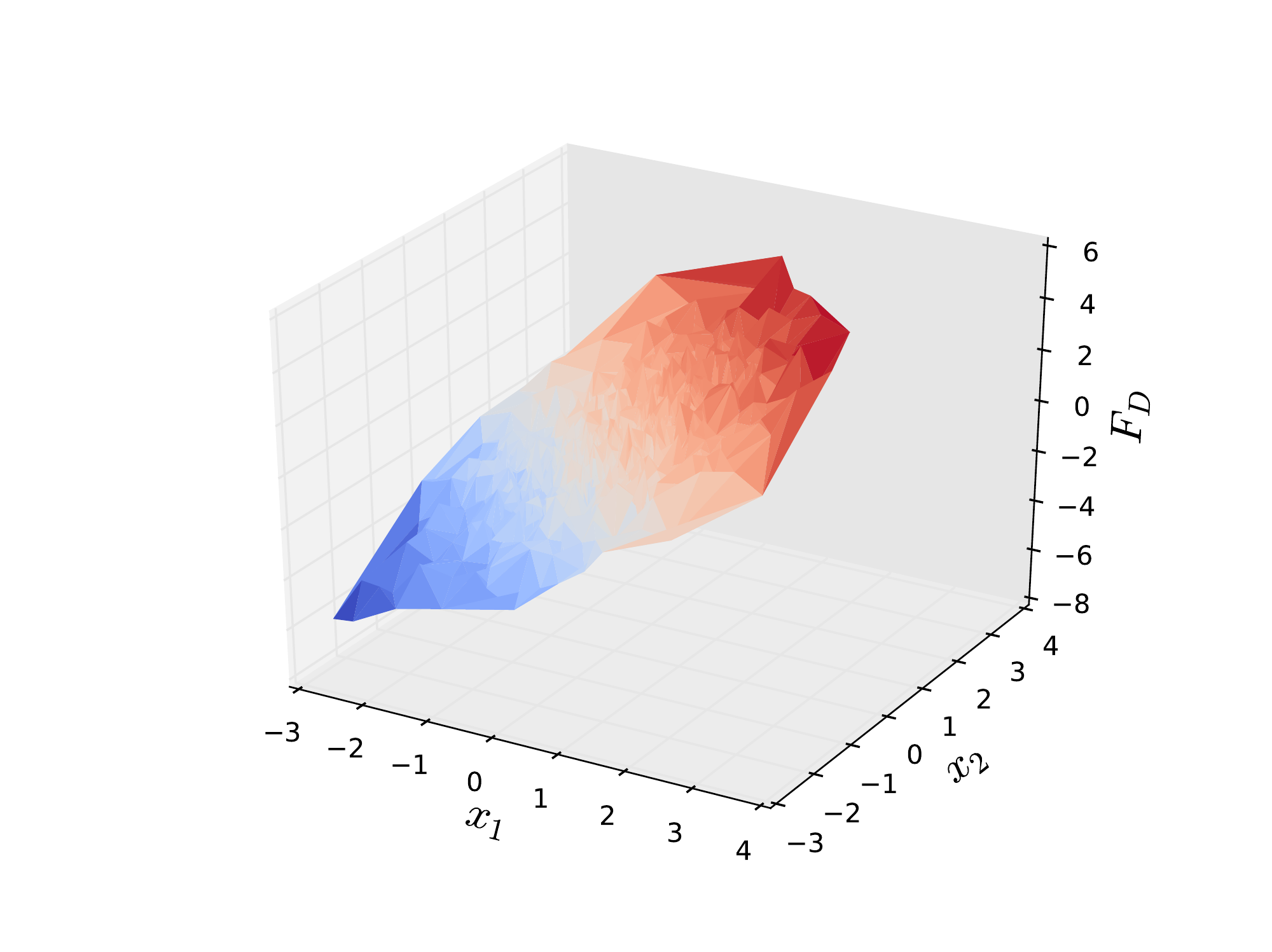}
    }
    \hfill
    \subfloat[$\lambda=2$]{%
    \includegraphics[width=0.5\linewidth]{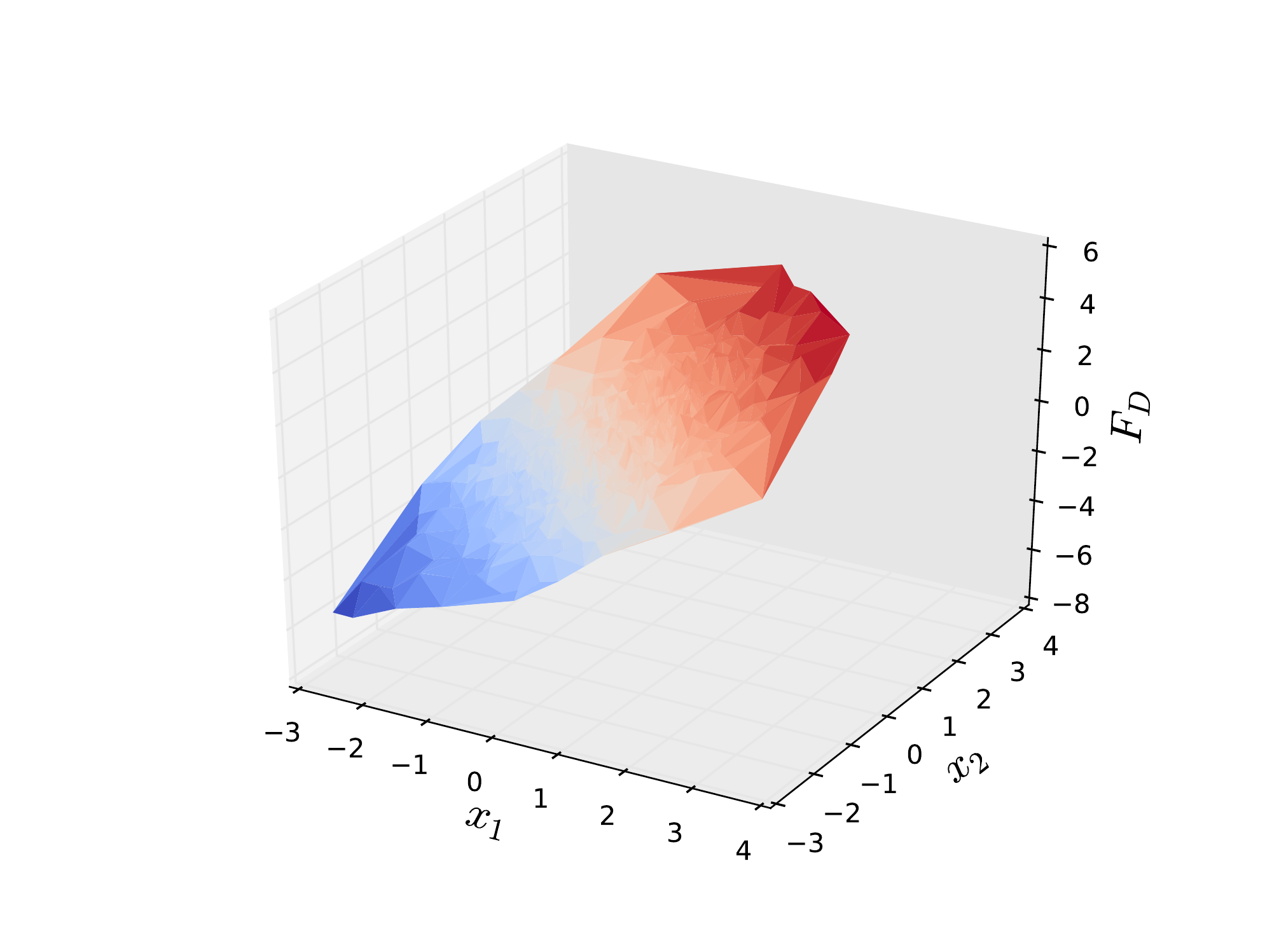}
    }
    \hfill
    \subfloat[$\lambda=10$]{%
    \includegraphics[width=0.5\linewidth]{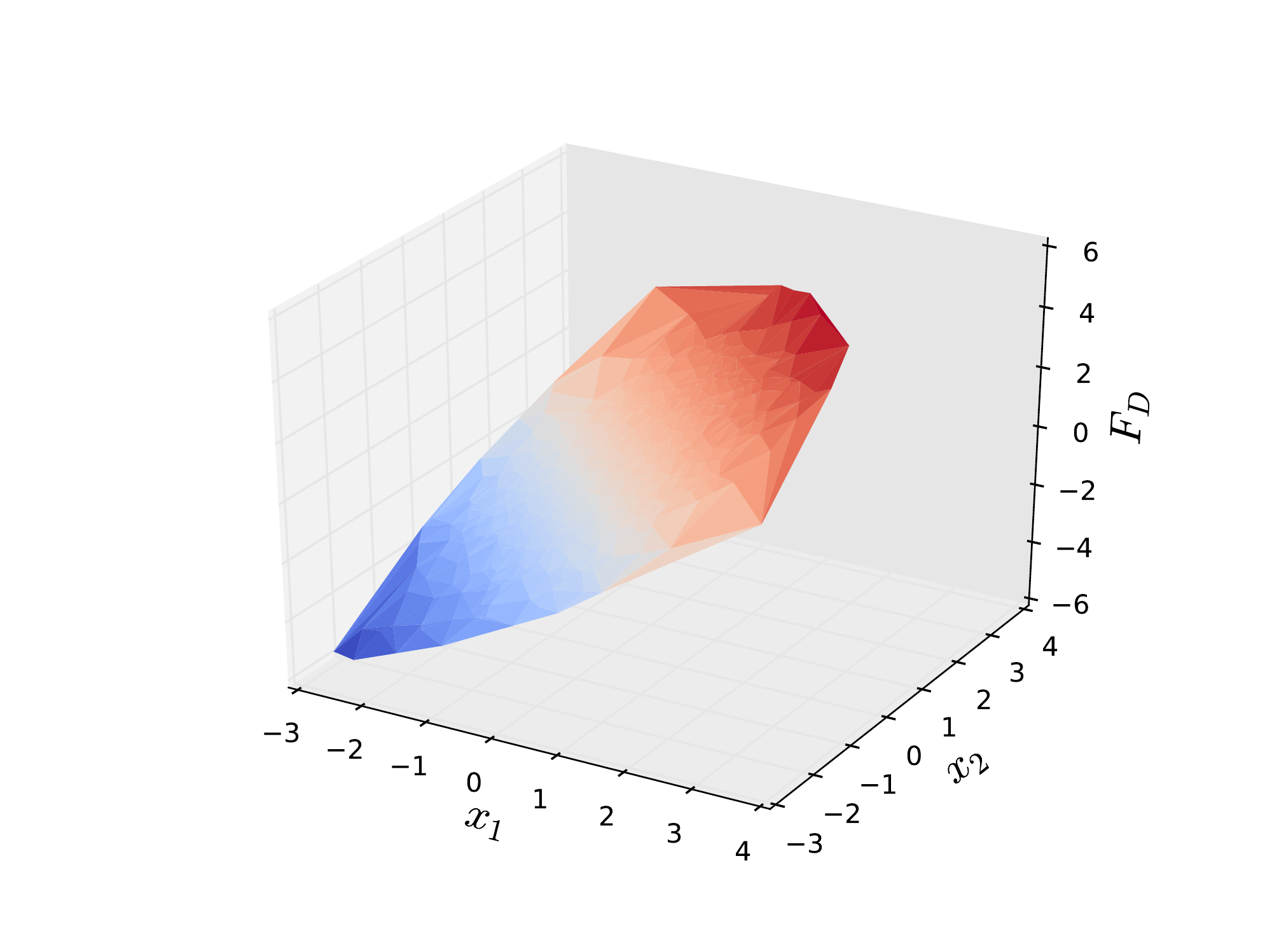}
    }
    \hfill
    \caption{Comparison of the smoothness of the regularized DTL using different values of $\lambda$, with the data from the linear model $y=X_1 + X_2 + \epsilon$. }
    \label{fig:Shrinkage_Behaviour}
\end{figure}

\subsection{Geometrical Optimality Properties}
We define a geometric loss function to analyze the geometric optimality of the Delaunay triangulation with random points.

\begin{definition}
Let $\mathcal{X}=\{\bold{X}_1, \ldots, \bold{X}_n\}$ $(n>p+1)$ be $n$ points in $R^p$ in general position with a convex hull $\mathcal{H}(\mathcal{X})$, and $\mathcal{T}$ be any triangulation of $\mathcal{X}$. For any point $\bold{X}\in R^p$, if $\bold{X}\in \mathcal{H}(\mathcal{X})$, the geometric loss function of the points is defined as
$$H(\bold{X}, \mathcal{X}, \mathcal{T})=\sum_{i=1}^{p+1}\lambda_i\|\bold{X}_{(i)}-\bold{X}\|^2, \text{ s.t.\ $\sum_{i=1}^{p+1} \lambda_i=1$, $\sum_{i=1}^{p+1} \lambda_i\bold{X}_{(i)}=\bold{X},$}$$
where $\bold{X}_{(i)}, i=1, \ldots, p+1$ are the vertices of the triangle in $\mathcal{T}$ that covers $\bold{X}$. Otherwise, if $\bold{X}\notin \mathcal{H}(\mathcal{X})$, we define $H(\bold{X}, \mathcal{X}, \mathcal{T})=0.$
\end{definition}

\begin{theorem}\label{geometric_optimal}
Let $\mathcal{X}=\{\bold{X}_1, \ldots, \bold{X}_n\}$ be i.i.d samples from a continuous density $f$, which is bounded away from zero and infinity on $[0,1]^p$.  Then, \begin{eqnarray*}
\mathbb{E}[H(\bold{X}, \mathcal{X}, \mathcal{D})]= \inf_{\mathcal{T}} \mathbb{E}[H(\bold{X}, \mathcal{X}, \mathcal{T})],
\end{eqnarray*}
where $\mathcal{D}$ is the Delaunay triangulation of $\mathcal{X}$.
\end{theorem}
\begin{proof}
For any point $\bold{X}\in R^p$ and a simplex with vertices $\mathcal{X}_{p+1}=\{\bold{X}_1, \ldots, \bold{X}_{p+1}\},$ if $\bold{X}\in \mathcal{H}(\mathcal{X}_{p+1})$, define the generalized geometric loss function as $$\tilde{H}(\bold{X}, \mathcal{X}_{p+1})=\sum_{i=1}^{p+1}\lambda_i\|\bold{X}_{i}-\bold{X}\|^2, \text{ s.t.\ $\sum_{i=1}^{p+1} \lambda_i=1$, $\sum_{i=1}^{p+1} \lambda_i\bold{X}_{i}=\bold{X}.$}$$ Otherwise, if $\bold{X}\notin \mathcal{H}(\mathcal{X}_{p+1}),$
$\tilde{H}(\bold{X}, \mathcal{X}_{p+1})=0.$
By Theorem 1 in \cite{Rajan1991}, among all the triangles with vertices in $\mathcal{X}$ that contain the point
$\bold{X}$, the Delaunay triangle minimizes the geometric loss function $\tilde{H}(\bold{X}, \mathcal{X}_{p+1})$ of the triangle at the
point. Thus, for any triangulation $\mathcal{T}$, we have
\begin{eqnarray*}
H(\bold{X}, \mathcal{X}, \mathcal{D})= \inf_{\mathcal{X}_{p+1}\subset  \mathcal{X}}\tilde{H}(\bold{X}, \mathcal{X}_{p+1})\leq H(\bold{X}, \mathcal{X}, \mathcal{T})
\end{eqnarray*}
Thus,
\begin{eqnarray*}
\mathbb{E}[H(\bold{X}, \mathcal{X}, \mathcal{D})]= \inf_{\mathcal{T}} \mathbb{E}[H(\bold{X}, \mathcal{X}, \mathcal{T})].
\end{eqnarray*}

\end{proof}

\subsection{Geometrical Asymptotic Properties}\label{geometricasymp}

Since the DTL makes prediction for a point by fitting a linear interpolation function locally on the Delaunay triangle that covers the point, the asymptotic behavior of the covering triangle is crucial to the asymptotics of DTL in both regression and classification problems.
We show the convergence rate of the size of the covering triangle in terms of its average edge length.

\begin{lemma}\label{lemma:1}
Let $\bold{X}, \bold{X}_1, \ldots, \bold{X}_n$ take values on $[0, 1]^p$. Let $\bold{X}_{(1)}$ denote the nearest neighbour of $\bold{X}$ among $\bold{X}_1, \ldots, \bold{X}_n$.
Then, for $p\geq 2$,
$$\mathbb{E}\|\bold{X}_{(1)}-\bold{X}\|^2 \leq c_p\left(\frac{1}{n+1}\right)^{2/p},$$
where $c_p=\frac{4(1+\sqrt{p})^2}{V_p^{2/p}}$ and $V_p$ is the volume of a unit $p$-ball; and
for $p=1$, it reduces to
$$\mathbb{E}\|\bold{X}_{(1)}-\bold{X}\|^2 \leq \frac{2}{n+1}.$$
\end{lemma}

\begin{proof}
Let $\bold{X}_{(i, 1)}$ denote the nearest neighbor of $\bold{X}_i$, and let $r_i = \|\bold{X}_{(i,1)}-\bold{X}_i\|.$
Define the $p$-ball $B_i=\left\{\bold{x}\in \mathbb{R}^p: \|\bold{x}-\bold{X}_i\|<r_i/2\right\}.$
Note that the $p$-balls are disjoint, and $r_i\leq \sqrt{p}$, and we have that
$\cup_{i=1}^{n+1}B_i\subseteq [-\frac{\sqrt{p}}{2}, 1+\frac{\sqrt{p}}{2}]^p$.
Thus,
$$\mu \left(\cup_{i=1}^{n+1}B_i \right)\leq (1+\sqrt{p})^p,$$
where $\mu$ is the Lebesgue measure.
Hence,
\begin{eqnarray}\label{eqn:1}
\sum_{i=1}^{n+1}V_p \left(\frac{r_i}{2}\right)^p\leq (1+\sqrt{p})^p.
\end{eqnarray}
For $p\geq 2$, it can be shown that
\begin{eqnarray*}
\left(\frac{1}{n+1}\sum_{i=1}^{n+1}r_i^2\right)^{p/2} &\leq& \frac{1}{n+1}\sum_{i=1}^{n+1}r_i^{p}\\
&\leq&\frac{1}{n+1}\times \frac{2^p(1+\sqrt{p})^p}{V_p},
\end{eqnarray*}
which proves the case of $p\geq 2$.
For $p=1$, it is obvious that
$$\frac{1}{n+1}\sum_{i=1}^{n+1}r_i^2\leq \frac{2}{n+1}.$$
\end{proof}

\begin{lemma}\label{lemma:2}
Let $\bold{X}, \bold{X}_1, \ldots, \bold{X}_n$ be i.i.d samples from a continuous density $f$, which is bounded away from zero and infinity on $[0,1]^p$. Let $H_n$ denote the event that $\bold{X}$ falls inside the convex hull of  $ \bold{X}_1, \ldots, \bold{X}_n$. Then, $\mathbb{P}(H_n)\rightarrow 1$, as $n\rightarrow \infty$.
\end{lemma}
\begin{proof}
Let $U_j(\delta), j=1, \ldots, 2^p$, denote a series of $\delta$-square-neighborhoods of the vertices of the $[0,1]^p$, and $\mathcal{X}=\{\bold{X}_1, \ldots, \bold{X}_n\}.$ Define the event $E_n=\cap_{j=1}^{2^p} \{U_j\cap \mathcal{X}\neq\varnothing\}$.
Then, given any arbitrary small value of $\delta>0$,  we have
\begin{eqnarray*}
\mathbb{P}(E_n^c)\leq\left(1-\min_{j\in\{1, \ldots, 2^p\}}\int_{U_j}f(\bold{x})d\bold{x}\right)^n
\rightarrow 0,
\end{eqnarray*}
which implies $\mathbb{P}({E_n})\rightarrow 1$, as $n \rightarrow \infty.$ Furthermore,
\begin{eqnarray*}
\mathbb{P}(H_n)&=&\mathbb{P}(H_n|E_n)\mathbb{P}(E_n) +\mathbb{P}(H_n|E_n^c)\mathbb{P}(E_n^c) \\
&\geq&\mathbb{P}(H_n|E_n)\mathbb{P}(E_n)\\
&\geq&(1-2\delta)^p\mathbb{P}(E_n),
\end{eqnarray*}
which is guaranteed by the convexity of the convex hull. Thus, $\mathbb{P}(H_n)\rightarrow1,$ as $n\rightarrow \infty.$
\end{proof}

\begin{theorem}\label{theorem:1}
Let $\bold{X}, \bold{X}_1, \ldots, \bold{X}_n$ be i.i.d samples from a continuous density $f$, which is bounded away from zero and infinity on $[0,1]^p$.  If the point $\bold{X}$ is inside the convex hull of $\bold{X}_1, \ldots, \bold{X}_n$, define $T(\bold{X})$ as the average length of the edges of the Delaunay triangle that covers $\bold{X}$; otherwise, $T(\bold{X})=\|\bold{X}-\bold{X}_{(1)}\|$, where $\bold{X}_{(1)}$ is the nearest neighbor point of $\bold{X}$ among $\bold{X}_1, \ldots, \bold{X}_n$. Then,
$$n^{1/p}\mathbb{E}T(\bold{X})\rightarrow \beta \int_{[0,1]^p}f(\bold{x})^{(p-1)/p}d\bold{x},$$
where $\beta$ is a positive constant depending only on the dimension $p$ and density $f$.
\end{theorem}

The proof of Theorem \ref{theorem:1} contains two parts. Following the machinary of the 1-NN proof \citep{Biau2015}, the first part shows that the convergence rate of $T(\bold{X})$ for the out-of-hull case is $\mathcal{O}(n^{-1/p})$, and since the point $\bold{X}$ falls inside the convex hull of $\bold{X}_1, \ldots, \bold{X}_n$ with probability one, the convergence rate is determined by the inside-hull case.  The second part of the proof follows the results of \cite{Jim2002}.

\begin{proof}[Proof of Theorem 2]
Let $H_n$ denote the event that point $\bold{X}$ falls inside the convex hull of $\bold{X}_1, \ldots, \bold{X}_n$, and by the definition of $T(\bold{X})$ we have
\begin{eqnarray*}
T(\bold{X})=I(H_n)\tilde{T}(\bold{X}) + (1-I(H_n)) \|\bold{X}-\bold{X}_{(1)}\|,
\end{eqnarray*}
where $\tilde{T}(\bold{X})$ is the average length of the edges of the Delaunay triangle that covers point $\bold{X}$, and if $\bold{X}$ is outside the convex hull, define $\tilde{T}(\bold{X})=0.$

We first show $\mathbb{E}\left\{(1-I(H_n)) \|\bold{X}-\bold{X}_{(1)}\|\right\}=o(n^{-1/p}).$
By adopting the techniques in \cite{Biau2015}, we have
\begin{eqnarray*}
\mathbb{E} \|\bold{X}-\bold{X}_{(1)}\|^2
&=&
\int_{0}^\infty \mathbb{P}\left\{\|\bold{X}-\bold{X}_{(1)}\|^2>t\right\}dt\\
&=&\int_0^\infty \mathbb{E} \{1-\mu(B(\bold{X}, \sqrt{t}))\}^{n}dt\\
&=& \frac{1}{n^{1/p}}\int_0^\infty \mathbb{E} \left\{1-\mu \left(B\left(\bold{X}, \frac{\sqrt{t}}{n^{1/p}}\right)\right)\right\}^{n}dt,
\end{eqnarray*}
where $B(\bold{X}, \sqrt{t})$ is a $p$-ball with center $\bold{X}$ and radius $\sqrt{t}$, and $\mu$ is the probability measure.
By the Lebesgue differentiation theorem, at Lebesgue-almost all $\bold{x}$,
$$\frac{\mu(B(\bold{x}, \rho))}{V_p\rho^p}\rightarrow f(\bold{x}), \quad {\rm as} \ \rho \downarrow 0,$$
where $V_p$ is the volume of a standard $p$-ball.
At such $\bold{x}$, we have for fixed $t$,
$$\left\{1-\mu \left(B\left(\bold{x}, \frac{\sqrt{t}}{n^{1/p}}\right)\right)\right\}^{n} \rightarrow \exp(-f(\bold{x})V_p t^{p/2}),\quad{\rm as}  \ n \rightarrow \infty.$$
Furthermore, Fatou's lemma implies
\begin{eqnarray*}
\lim_{n\rightarrow \infty} \inf n^{2/p} \int_{0}^\infty \mathbb{P}\left\{\|\bold{X}-\bold{X}_{(1)}\|^2>t\right\}dt&\geq&
\int_{[0,1]^p}\int_0^\infty \exp(-f(\bold{x})V_p t^{p/2})f(\bold{x})dtd\bold{x}\\
&=& \Gamma(2/p+1)\int_{[0,1]^p:f>0}\frac{f(\bold{x})}{(f(\bold{x})V_p)^{2/p}}d\bold{x}\\
&=& \frac{\Gamma(2/p+1)}{V_p^{2/p}}\int_{[0,1]^p:f>0}{f(\bold{x})^{1-2/p}}d\bold{x}.
\end{eqnarray*}
To establish the upper bound for $p>2$, we take an arbitrarily large constant $T_0$ and split
\begin{eqnarray*}
n^{2/p}\int_{0}^\infty \mathbb{P}\left\{\|\bold{X}-\bold{X}_{(1)}\|^2>t\right\}dt&=&
n^{2/p}\mathbb{E}\left\{\|\bold{X}-\bold{X}_{(1)}\|^2 I(\|\bold{X}-\bold{X}_{(1)}\| \leq T_0/ n^{1/p} )\right\} \\
&&+n^{2/p}\mathbb{E}\left\{\|\bold{X}-\bold{X}_{(1)}\|^2 I(\|\bold{X}-\bold{X}_{(1)}\|> T_0/ n^{1/p} )\right\} \\
&=& \bold{I} + \bold{II}.
\end{eqnarray*}
By Fatou's lemma, the first term $\bold{I}$ can be bounded as follows,
\begin{eqnarray*}
&&\lim_{n\rightarrow \infty} \sup n^{2/p} \mathbb{E}\left\{\|\bold{X}-\bold{X}_{(1)}\|^2 I(T(\bold{X})\leq T_0/n^{1/p})\right\}\\
&\leq& \int _{ [0,1]^p} \int_{0}^{T_0} \exp(-f(\bold{x})V_p t^{p/2})f(\bold{x})dt d\bold{x}\\
&\leq& \frac{\Gamma(2/p+1)}{V_p^{2/p}}\int_{[0,1]^p}{f(\bold{x})^{1-2/p}}d\bold{x}.
\end{eqnarray*}
Then, we show that the second term $\bold{II}$ is controlled by the choice of $T_0$.
By the symmetry of the nearest neighbor distance, we can rewrite $\bold{II}$ as
\begin{eqnarray*}
\bold{II} = \frac{n^{2/p}}{n+1} \mathbb{E}\left\{\sum_{i=1}^{n+1} \|\bold{X}_i-\bold{X}_{(i, 1)}\|^2 I( \|\bold{X}_i-\bold{X}_{(i, 1)}\|>T_0/n^{1/p})\right\}.
\end{eqnarray*}
Denote $r_i=\|\bold{X}_i-\bold{X}_{(i, 1)}\|$, and it has been proved in Lemma \ref{lemma:1} that
$$\sum_{i=1}^{n+1} r_i^p\leq a_p=\frac{2^p(1+\sqrt{d})^p}{V_p}.$$
Set $K=\sum_{i=1}^{n+1}  I(r_i>T_0/n^{1/p})$, and we know that $K\leq n a_p/T_0^p$.
By Jensen's inequality, when $K>0$,
$$\left(\frac{1}{K}\sum_{i=1}^{n+1} r_i^2 I(r_i>T_0/n^{1/p})\right)^{p/2}\leq
\frac{1}{K}\sum_{i=1}^{n+1}r_i^p I(r_i>T_0/n^{2/p})\leq \frac{a_p}{K}.$$
Thus,
\begin{eqnarray*}
\bold{II}&=&
\frac{n^{2/p}}{n+1}\mathbb{E}\left[I(K>0)\sum_{i=1}^{n+1}r_i^2 I(r_i>T_0/n^{1/p})\right]\\
&\leq&\frac{n^{2/p}}{n}\mathbb{E}\left[K\left(a_p/K\right)^{2/p}I(K>0)\right]\\
&=&a_p^{2/p}\mathbb{E}\left(\frac{K}{n}\right)^{1-2/p}\\
&\leq& \frac{a_p}{T_0^{p-2}}.
\end{eqnarray*}
As a result, we have shown for $p>2$,
\begin{eqnarray}
n^{2/p}\mathbb{E} \|\bold{X}-\bold{X}_{(1)}\|^2 \rightarrow \frac{\Gamma(2/p+1)}{V_p^{2/p}}\int_{[0,1]^p}f^{1-2/p}(\bold{x})d\bold{x}.
\end{eqnarray}
For $p=2$, (\ref{eqn:1}) leads to
\begin{eqnarray*}
\pi \sum_{i=1}^{n+1}r_i^2\leq 4(1+\sqrt{2})^2.
\end{eqnarray*}
As $n\mathbb{E}\|\bold{X}_{(1)} - \bold{X}\|^2 = \frac{n}{n+1}\mathbb{E}(\sum_{i=1}^{n+1}r_i^2),$
we have $n\mathbb{E}\|\bold{X}_{(1)} - \bold{X}\|^2<\frac{4}{\pi}(1+\sqrt{2})^2.$
By the Cauchy--Schwarz inequality,
\begin{eqnarray*}
\mathbb{E}\{(1-I(H_n)) \|\bold{X}-\bold{X}_{(1)}\|\}
&\leq& \sqrt{\mathbb{E}(1-I(H_n))^2 \mathbb{E} \|\bold{X}-\bold{X}_{(1)}\|^2}\\
&=&  \sqrt{(1-\mathbb{P}(H_n))} \times \sqrt{\mathbb{E} \|\bold{X}-\bold{X}_{(1)}\|^2}\\
&= &\sqrt{(1-\mathbb{P}(H_n))} \times O(n^{-1/p}).
\end{eqnarray*}

As shown in Lemma \ref{lemma:2},
$\mathbb{P}(H_n) \rightarrow 1$, as $n \rightarrow \infty$.
Thus, $\mathbb{E}\{(1-I(H_n)) \|\bold{X}-\bold{X}_{(1)}\|\}=o(n^{-1/p}).$

Next, we show that $n^{1/p}\mathbb{E}\{I(H_n)\tilde{T}(\bold{X})\}\rightarrow\beta \int_{[0,1]^p}f(\bold{x})^{(p-1)/p}d\bold{x}.$
Let $N(n)$ denote the number of triangles in the Delaunay triangulation of $\bold{X}_1, \ldots, \bold{X}_n,$ and let $\Delta_1, \ldots, \Delta_N(n)$ denote the triangles.
By definition, we show that
$\mathbb{E}\tilde{T}(\bold{X}) = \frac{1}{2N(n)}\sum_{i=1}^{N(n)}\sum_{e\in S(\bold{x}_i)} |e|,$
where $S(\bold{x}_i)$ is the set of edges in the Delaunay graph that are connected with point $\bold{x}_i$, and $|e|$ is the length of edge $e$.
As shown in Lemma \ref{lemma:2},
\begin{eqnarray*}
\sum_{i=1}^{N(n)}\int_{\Delta_{i}}f(\bold{x})d\bold{x}=\mathbb{P}(H_n) \rightarrow 1,
\end{eqnarray*}
as $n\rightarrow \infty.$
By the symmetry of triangle indices,
we have
$$N(n)\int_{\Delta_{1}}f(\bold{x})d\bold{x} \rightarrow 1.$$
Thus, we have
$N(n)\mathbb{P}(\bold{X}\in \Delta_1)\rightarrow1.$
Hence, by reindexing the points of $\bold{X}_1, \ldots, \bold{X}_n,$
\begin{eqnarray*}
\mathbb{E}\tilde{T}(\bold{X}) &=&\sum_{j=1}^{N(n)} \mathbb{P}(\bold{X}\in \Delta_j){{p+1}\choose{2}}^{-1}\sum_{e\in\Delta_{j}}|e|\\
&\rightarrow&\sum_{j=1}^{N(n)} \frac{1}{N(n)}{{p+1}\choose{2}}^{-1}\sum_{e\in\Delta_{j}}|e|\\
&=&\frac{1}{2n}\sum_{i=1}^{n}\sum_{e\in S(\bold{x}_i)} |e|,
\end{eqnarray*}
where $S(\bold{x}_i)$ denotes the set of edges that are connected with point $\bold{x}_i.$

By Theorem 2.1 in \cite{Jim2002}, it is shown that for the Delaunay graph,
$$n^{-1+c/p}\mathbb{E}\left\{\frac{1}{2}\sum_{i=1}^n\sum_{e\in S(\bold{x}_i)} |e|^c\right\}\rightarrow \frac{1}{2}\mathbb{E}\left(\sum_{i=1}^{m}A_i^c\right)\int_{[0,1]^p}f(\bold{x})^{(p-c)/p}d\bold{x},$$
where $A_i$'s are the distances between the points $\bold{s}_i, i=1,\ldots, m$, to the original point $\bold{0},$ and
$\bold{s}_i, i=1,\ldots, m$, are from a $p$-dimensional unit intensity Poisson process with density $f(\bold{x}).$

With $p=1$ and since $I(H_n)\rightarrow1$ in probability, we obtain
$$n^{1/p}\mathbb{E}\left\{I(H_n)\tilde{T}(\bold{X})\right\}\rightarrow \beta\int_{[0,1]^p}f(\bold{x})^{(p-1)/p}d\bold{x},$$
where $\beta$ is a constant that depends only on dimension $p$ and density $f$.
\end{proof}

\begin{corollary}\label{corollary:1}
Let $\bold{X}, \bold{X}_1, \ldots, \bold{X}_n$ be i.i.d samples from a continuous density $f$, which is bounded away from zero and infinity on $[0,1]^p$. If the point $\bold{X}$ is inside the convex hull of $\bold{X}_1, \ldots, \bold{X}_n$, define $\bold{X}_{(1)}$ as any one of the vertices of the covering Delaunay triangle; otherwise, define $\bold{X}_{(1)}$ as the nearest neighbor point of $\bold{X}.$ Then,
$\bold{X}_{(1)}\rightarrow \bold{X}$ in probability.
\end{corollary}

\begin{proof}
We need to show that $\|\bold{X}_{(1)}-\bold{X}\| \leq \tilde{T}(\bold{X}).$
If point $\bold{X}$ is outside the convex hull of $\bold{X}_1, \ldots, \bold{X}_n$, $\|\bold{X}_{(1)}-\bold{X}\|=\tilde{T}(\bold{X})$, by definition. Otherwise,
$\|\bold{X}-\bold{X}_{(1)}\|\leq \max_j \|\bold{X}_{(1)}-\bold{X}_{(j)}\|\leq {{p+1}\choose{2}}\tilde{T}(\bold{X}).$
From Theorem \ref{theorem:1}, we conclude that
$\mathbb{E}\|\bold{X}_{(1)}-\bold{X}\| \rightarrow 0$, as $n\rightarrow \infty,$
which implies that $\bold{X}_{(1)}\rightarrow \bold{X}$ in probability.
\end{proof}
This result shows that the vertices of the covering simplex of point $\bold{X}$ all converge to $\bold{X}$ in probability.

\subsection{General Case} \label{regclass}
For regression problems, the DTL is shown to be consistent, and for classification problems,
the DTL has an error rate smaller than or equal to $2R_B(1-R_B)$,
where $R_B$ is the Bayes error rate (the minimum error rate that can be achieved by any function approximator).
\subsubsection*{Regression}
\begin{theorem}
Assume that $\mathbb{E}y^2<\infty$, and
$\psi(\bold{X})=\mathbb{E}\left(y|\bold{X}\right)$ is continuous. Then, the DTL regression function estimate $\hat{F}_D$ satisfies
\begin{eqnarray*}
\mathbb{E}|\hat{F}_D(\bold{X})-\psi(\bold{X})|^2\rightarrow L^*,
\end{eqnarray*}
where $L^*=\inf_g \mathbb{E}|y-g(\bold{X})|^2$ is the minimal value of the $L_2$ risk over all continuous functions $g: \mathbb{R}^p \rightarrow \mathbb{R}$.
\end{theorem}
\begin{proof}%[Proof of Theorem 3]
If point $\bold{X}$ is outside the convex hull $\mathcal{H}_n$, $\hat{F}_D(\bold{X})=y_{(1)}$, where $y_{(1)}$ is the corresponding response of the nearest neighbor point $\bold{X}_{(1)}.$
If $\bold{X}$ is inside $\mathcal{H}_n$, denote the simplex covering $\bold{X}$ as $\Delta_\bold{X},$ and the vertices of the simplex $\Delta_\bold{X}$ as $\bold{x}(\Delta_\bold{X})_{[1]}, \ldots, \bold{x}(\Delta_\bold{X})_{[p+1]},$ or simply $\bold{x}_{[1]}, \ldots, \bold{x}_{[p+1]}.$ Thus, for any interior point $\bold{x}$ of simplex $\Delta_\bold{X},$ it can be represented as $\bold{x}=\sum_{j=1}^{p+1}\lambda_{[j]}(\bold{X})\bold{x}_{[j]},$ where $\lambda_{[j]}(\bold{X})>0$, for $j=1,\ldots,p+1$, and $\sum_{j=1}^{p+1}{\lambda_{[j]}(\bold{X})}=1.$
As the DTL is linear inside the simplex, and $\hat{F}_D(\bold{x})=\sum_{j=1}^{p+1}\lambda_{[j]}(\bold{X})y_{[j]},$ where $y_{[j]}$ is the response to corresponding point $\bold{x}_{[j]}.$
In general,
$\hat{F}_D(\bold{X})=I_n\sum_{j=1}^{p+1}\lambda_{[j]}(\bold{X})y_{[j]} + (1-I_n)y_{(1)},$ with $I_n=I(H_n)$.

Thus, we have
\begin{eqnarray}\label{eqn:3terms}
&&\mathbb{E}|\hat{F}_D(\bold{X})-\psi(\bold{X})|^2 \nonumber\\
&=&\mathbb{E}\bigg|I_n\sum_{j=1}^{p+1}\lambda_{[j]}(\bold{X})y_{[j]} + (1-I_n)y_{(1)}-\psi(\bold{X})\bigg|^2\nonumber\\
&=&\mathbb{E}\bigg|I_n\sum_{j=1}^{p+1}\lambda_{[j]}(\bold{X})y_{[j]} + (1-I_n)y_{(1)}-I_n\sum_{j=1}^{d+1}\lambda_{[j]}(\bold{X})\psi(\bold{X}_{[j]})+I_n\sum_{j=1}^{p+1}\lambda_{[j]}(\bold{X})\psi(\bold{X}_{[j]})-\psi(\bold{X})\bigg|^2\nonumber\\
&=&\mathbb{E}\bigg|I_n\sum_{j=1}^{p+1}\lambda_{[j]}(\bold{X})y_{[j]} + (1-I_n)y_{(1)}-I_n\sum_{j=1}^{p+1}\lambda_{[j]}(\bold{X})\psi(\bold{X}_{[j]})\bigg|^2+\mathbb{E}\bigg|I_n\sum_{j=1}^{p+1}\lambda_{[j]}(\bold{X})\psi(\bold{X}_{[j]})-\psi(\bold{X})\bigg|^2\nonumber\\
\hspace{-5mm}&&+2\mathbb{E}\left[ \left(I_n\sum_{j=1}^{p+1}\lambda_{[j]}(\bold{X})y_{[j]} + (1-I_n)y_{(1)}-I_n\sum_{j=1}^{p+1}\lambda_{[j]}(\bold{X})\psi(\bold{X}_{[j]}) \right)\left(I_n\sum_{j=1}^{p+1}\lambda_{[j]}(\bold{X})\psi(\bold{X}_{[j]})-\psi(\bold{X})\right)\right]
\end{eqnarray}
From Corollary  \ref{corollary:1}, we know that $\bold{X}_{[j]}\rightarrow \bold{X}$ for each $j=1, \ldots, p+1$, in probability.
By the continuity of function $\psi(\bold{x})$, the second term in (\ref{eqn:3terms}) goes to zero, i.e., $\mathbb{E}\big|I_n\sum_{j=1}^{p+1}\lambda_{[j]}(\bold{X})\psi(\bold{X}_{[j]})-\psi(\bold{X})\big|^2\rightarrow0.$ For the third term in (\ref{eqn:3terms}), the Cauchy--Schwarz inequality implies
\begin{eqnarray*}
&&\mathbb{E}\left[ \left(I_n\sum_{j=1}^{p+1}\lambda_{[j]}(\bold{X})y_{[j]} + (1-I_n)y_{(1)}-I_n\sum_{j=1}^{p+1}\lambda_{[j]}(\bold{X})\psi(\bold{X}_{[j]}) \right)\left(I_n\sum_{j=1}^{p+1}\lambda_{[j]}(\bold{X})\psi(\bold{X}_{[j]})-\psi(\bold{X})\right)\right]\\
&\leq&\mathbb{E} \bigg|I_n\sum_{j=1}^{p+1}\{\lambda_{[j]}(\bold{X})y_{[j]} -\psi(\bold{X}_{[j]})\}+ (1-I_n)y_{(1)} \bigg|^2\times\mathbb{E} \bigg|I_n\sum_{j=1}^{p+1}\lambda_{[j]}(\bold{X})\psi(\bold{X}_{[j]})-\psi(\bold{X})\bigg|^2\\
&=& o(1).
\end{eqnarray*}
As $n\rightarrow \infty$, $\mathbb{E} (I_n)\rightarrow1$, then
$$\mathbb{E} \big|I_n\sum_{j=1}^{p+1}\lambda_{[j]}(\bold{X})\{y_{[j]} -\psi(\bold{X}_{[j]})\}+ (1-I_n)y_{(1)} \big|^2\rightarrow \mathbb{E} \big|\sum_{j=1}^{p+1}\lambda_{[j]}(\bold{X})\{y_{[j]} -\psi(\bold{X}_{[j]})\} \big|^2<\infty,$$ and $\mathbb{E} \big|I_n\sum_{j=1}^{p+1}\lambda_{[j]}(\bold{X})\psi(\bold{X}_{[j]})-\psi(\bold{X})\big|^2=o(1).$

For the first term in (\ref{eqn:3terms}), we have
$$\mathbb{E}\bigg|\sum_{j=1}^{p+1}\lambda_{[j]}(\bold{X})y_{[j]}-\sum_{j=1}^{p+1}\lambda_{[j]}(\bold{X})\psi(\bold{X}_{[j]})\bigg|^2\leq\mathbb{E} \sum_{j=1}^{p+1}\big|y_{[j]}-\psi(\bold{X}_{[j]}) \big|^2$$
and again by the continuity of $\psi(\bold{x})$,  $$\mathbb{E} \sum_{j=1}^{p+1}|y_{[j]}-\psi(\bold{X}_{[j]}) |^2\rightarrow
\mathbb{E} \sum_{j=1}^{p+1}|y-\psi(\bold{X}) |^2=L^*.$$
Since $L^*$ is the minimum value of $\mathbb{E}|\hat{F}_D(\bold{X})-\psi(\bold{X})|^2,$
the consistency is shown.
\end{proof}

\subsubsection*{Classification}
\begin{theorem}
For a two-class classification model, if the conditional probability $\psi(\bold{x})=\mathbb{P}\left(y=1|\bold{x}\right)$ is a continuous function of $\bold{x}$. The mis-classification risk $R$ of a DTL classifier is bounded as
\begin{eqnarray*}
R_B \leq R \leq 2R_B(1-R_B),
\end{eqnarray*}
where $R_B$ is the Bayes error of the model.
\end{theorem}
\begin{proof}%[Proof of Theorem 4]
By Lemma \ref{lemma:2}, the point $\bold{X}$ is covered by a Delaunay triangle with probability one. Denote the vertices of the covering triangle as $\bold{X}_{[1]}, \ldots, \bold{X}_{[p+1]}.$
Define $r_n(\bold{X})$ as the empirical risk of the DTL classifier,
\begin{eqnarray*}
r_n(\bold{X})&=&\mathbb{P}(y\neq \hat{y}|\bold{X}, \bold{X}_{[1]}, \ldots, \bold{X}_{[p+1]})\\
&=&\mathbb{P}(y=1|\bold{X})\mathbb{P}(\hat{y}=0| \bold{X}_{[1]}, \ldots, \bold{X}_{[p+1]}) +\mathbb{P}(y=0|\bold{X})\mathbb{P}(\hat{y}=1| \bold{X}_{[1]}, \ldots, \bold{X}_{[p+1]})\\
&=&\psi(\bold{X})\left(1-\sum_{j=1}^{p+1}\lambda_{[j]}(\bold{X})\psi( \bold{X}_{[j]})\right) +(1-\psi(\bold{X}))\sum_{j=1}^{p+1}\lambda_{[j]}(\bold{X})(1-\psi( \bold{X}_{[j]})),
\end{eqnarray*}
where $\hat{y}=I(\hat{F}_D(\bold{X})>1/2)$ is the prediction of the DTL classifier, and $\sum_{j=1}^{p+1}\lambda_{[j]}=1.$
Since Corollary \ref{corollary:1} shows that for any $j$, $\bold{X}_{[j]}\rightarrow\bold{X}$ in probability, based on the continuity of $\psi(\cdot)$, we have $\psi( \bold{X}_{[j]}) \rightarrow \psi(\bold{X})$ in probability.
Thus, $r_n(\bold{X})\rightarrow2\psi(\bold{X})(1-\psi(\bold{X}))$, in probability.
The conditional Bayes risk
$r^*(\bold{X})=\min\{\psi(\bold{X}), 1-\psi(\bold{X})\}.$
By the symmetry of function $r^*$ in $\psi(\cdot)$, we write
$$r(\bold{X})=2\psi(\bold{X})(1-\psi(\bold{X}))=2r^*(\bold{X})(1-r^*(\bold{X})).$$
By definition, the DTL risk $R=\lim_{n \rightarrow \infty}\mathbb{E}\{r_n(\bold{X})\},$ and since
$r_n(\cdot)$ is bounded by 1,
applying the dominated convergence theorem,
$R=\mathbb{E}\{\lim_{n\rightarrow \infty} r_n(\bold{X})\}.$
The limit yields
\begin{eqnarray*}
R&=&\mathbb{E}\{r(\bold{X})\}
=\mathbb{E}\{2r^*(\bold{X})(1-r^*(\bold{X}))\}.
\end{eqnarray*}
As the Bayes risk is the expectation of $r^*$,
 we have
$R=2R_B(1-R_B)-2{\rm Var} \{r^*(\bold{X})\},$
and thus, $R\leq 2R_B(1-R_B).$
\end{proof}

\subsection{Smooth and Noiseless Case}\label{smooth}

As we consider the local behavior at $\bold{x}$, without loss of generality, we assume $\bold{x}=\bold{0}$.

\begin{lemma}\label{independent_uniform}
Assume that $\bold{0}$ is a Lebesgue point of $f$, $f(\bold{0})>0$. Denote $\bold{X}_{1}, \ldots ,\bold{X}_{p+1}$ as the vertices of the Delaunay triangle that covers the point $\bold{0}$, and further let $\bold{X}_{(1)}, \ldots ,\bold{X}_{(p+1)}$ denote these vertices ordered by their distances from $\bold{0}$.
Then,
$$(f(\bold{0})V_{p} n)^{1/p}(\bold{X}_{(1)}, \ldots ,\bold{X}_{(p+1)}) \xrightarrow[]{\mathcal{D}} (\bold{Z}_1G_1^{1/p}, \ldots,
\bold{Z}_{p+1}G_{p+1}^{1/p}),$$
where $G_i=\sum_{j=1}^i E_j$, and $E_1, \ldots, E_{p+1}$ are independent standard exponential random variables, and $\bold{Z}_1, \ldots, \bold{Z}_{p+1}$ are independent random vectors uniformly distributed on the surface of $B(\bold{0}, 1)$, a $p$-ball with center $\bold{0}$ and radius $1$.
\end{lemma}

\begin{proof}
Let $K$ be a positive constant to be chosen later.
Consider a density $g_{K,n}$ related to $f$ as follows: let
$$d=\int_{B\left(\bold{0}, \frac{K}{n^{1/{p}}}\right)}f(\bold{z})d\bold{z},$$
and set
\begin{equation*}
    X=
    \begin{cases}
      f(\bold{0}), & \text{for } \bold{x} \in B(\bold{0}, \frac{K}{n^{1/p}}) \\
      f(\bold{x})\left(\frac{1-f(\bold{0})V_p\frac{K^p}{n}}{1-d}  \right), & \text{otherwise,}
    \end{cases}
  \end{equation*}
  which is a proper density (i.e., nonnegative and integrating to one) for $n$ large enough. Note that
  \begin{eqnarray*}
\int_{R^d}|g_{K, n}(\bold{x})-f(\bold{x})|d\bold{x}&=&\int_{B\left(\bold{0}, \frac{K}{n^{1/p}}\right)}|f(\bold{0})-f(\bold{x})|d\bold{x}\\
&&+\int_{B^c\left(\bold{0}, \frac{K}{n^{1/p}}\right)}f(\bold{x})\bigg|\frac{1-f(\bold{0})V_p \rho^p}{1-d}-1\bigg|d\bold{x}\\
&=&\int_{B\left(\bold{0}, \frac{K}{n^{1/p}}\right)}|f(\bold{0})-f(\bold{x})|d\bold{x}+\bigg|d-f(\bold{0})V_p\frac{K^p}{n}\bigg|\\
&\leq& 2\int_{B\left(\bold{0}, \frac{K}{n^{1/p}}\right)}|f(\bold{0})-f(\bold{x})|d\bold{x}\\
&=& o\left(\frac{1}{n}\right).
  \end{eqnarray*}
Therefore, by Doeblin's coupling method, there exist random variables $\bold{X}$ and $\bold{Y}$ with corresponding densities $\bold{f}$ and $g_{K,n}$, such that
  $\mathbb{P}(\bold{Y}\neq\bold{X})=\frac{1}{2}\int_{\mathbb{R}^p}|g_{K, n}(\bold{x})-f(\bold{x})|d\bold{x}=o(\frac{1}{n}).$
  \end{proof}
Repeating this $n$ times, we create two coupled samples of random variables that are $i.i.d.$ within the sample.
The sample $\bold{X}_1, \ldots, \bold{X}_n$ is drawn from the distribution of $\bold{X}$, and the sample $\bold{Y}_1, \ldots, \bold{Y}_n$ is drawn from the distribution of $\bold{Y}$.
Recall that the total variation distance between two random vectors $\bold{W}$, $\bold{W}^{'} \in R^n$ is defined by
$d_{\rm TV}(\bold{W}, \bold{W}^{'})=\sup_{A\in\mathcal{B}}|\mathbb{P}\{\bold{W}\in A\}-\mathbb{P}\{\bold{W}^{'}\in A\}|,$
where $\mathcal{B}$ denotes the Borel sets of $\mathbb{R}^p.$ Let $\|\bold{Y}_{(1)} \|\leq \ldots \leq \|\bold{Y}_{(n)}\|$ and $\|\bold{X}_{(1)} \|\leq \ldots \leq \|\bold{X}_{(n)}\|$ be the reordered samples. Then

\begin{eqnarray*}
d_{\rm TV}((\bold{Y}_{(1)}, \ldots, \bold{Y}_{(p+1)}), (\bold{X}_{(1)}, \ldots, \bold{X}_{(p+1)}))
&\leq& d_{\rm TV}((\bold{Y}_{1}, \ldots, \bold{Y}_{n}), (\bold{X}_{1}, \ldots, \bold{X}_{n}))\\
&\leq& \sum_{i=1}^{n}\mathbb{P} \{\bold{Y}_i\neq \bold{X}_i\}\\
&\leq& n\times o\left(\frac{1}{n}\right)\\
&=&o(1).
\end{eqnarray*}
Define
\begin{eqnarray*}
U_{(i)} =\nu (B(\bold{0}, \|\bold{Y}_{(i)}\|)),
\end{eqnarray*}
where $\nu$  is the probability measure of $\bold{Y}$. We recall that $U_{(1)}, \ldots, U_{(p+1)}$ are uniform
order statistics, and thus
$n(U_{(1)}, \ldots, U_{(p+1)})\xrightarrow[]{\mathcal{D}}(G_1, \ldots, G_{p+1}).$
In fact, this convergence is also in the $d_{\rm TV}$ sense.
Also, if
$\|Y_{(p+1)}\|\leq K/n^{1/p}$, then
$U_{(i)}=f(\bold{0})V_p\|\bold{Y}_{(i)}\|^p.$
Thus,
\begin{eqnarray*}
&&d_{\rm TV}\left(\left(\|\bold{Y}_{(1)}\|, \ldots, \|\bold{Y}_{(p+1)}\|\right),\left(\left(\frac{U_{(1)}}{f(\bold{0})V_p}\right)^{1/p}, \ldots, \left(\frac{U_{(p+1)}}{f(\bold{0})V_p}\right)^{1/p}\right)    \right)\\
&\leq&\mathbb{P}\left\{\|\bold{Y}_{(p+1)} \|>\frac{K}{n^{1/p}} \right\}\\
&=& \mathbb{P}\left\{ {\rm Bin}(n, f(\bold{0})V_p\frac{K^p}{n})<p+1\right\}\\
&=& \mathbb{P}\left\{ {\rm Poisson}(f(\bold{0})V_p K^p)<p+1\right\} + o(1)\\
&\leq& \epsilon + o(1)
\end{eqnarray*}
for all $K$ large enough, depending on $\epsilon.$
Let $\bold{Z}_1, \ldots, \bold{Z}_{p+1}$ be $i.i.d.$ random vectors
uniformly distributed on the surface of
$B(\bold{0}, 1).$ Then,
\begin{eqnarray*}
&&d_{\rm TV}\left((f(\bold{0})V_p n)^{1/p}(\bold{X}_{(1)}, \ldots, \bold{X}_{(p+1)}), (\bold{Z}_1G_1^{1/p}, \ldots, \bold{Z}_{p+1}G_{p+1}^{1/p})\right)\\
&\leq&d_{\rm TV} \left((\bold{X}_{(1)}, \ldots, \bold{X}_{(p+1)}), (\bold{Y}_{(1)}, \ldots, \bold{Y}_{(p+1)})\right)\\
&&+ d_{\rm TV}\left((f(\bold{0})V_p n)^{1/p}(\bold{Y}_{(1)}, \ldots, \bold{Y}_{(p+1)}), (\bold{Z}_1G_1^{1/p}, \ldots, \bold{Z}_{p+1}G_{p+1}^{1/p})\right)\\
&\leq& o(1) + d_{\rm TV} \left((\bold{Y}_{(1)}, \ldots, \bold{Y}_{(p+1)}), (\bold{Z}_1\|\bold{Y}_{(1)}\|, \ldots, \bold{Z}_{p+1}\|\bold{Y}_{(p+1)}\|)\right)\\
&&+ d_{\rm TV}\left((f(\bold{0})V_p n)^{1/p}(\bold{Y}_{(1)}, \ldots, \bold{Y}_{(p+1)}), (\bold{Z}_1G_1^{1/p}, \ldots, \bold{Z}_{p+1}G_{p+1}^{1/p})\right)\\
&\leq& o(1) + \mathbb{P}\left\{\|\bold{Y}_{p+1} \|>\frac{K}{n^{1/p}}\right\}\\
&&+ d_{\rm TV}\left(\left(\| \bold{Y}_{(1)}\|, \ldots, \| \bold{Y}_{(p+1)}\|\right), \left(\left(\frac{U_{(1)}}{f(\bold{0})V_p}\right)^{1/p}, \ldots, \left(\frac{U_{(p+1)}}{f(\bold{0})V_p}\right)^{1/p}\right)\right)\\
&&+ d_{\rm TV}\left(n^{1/p}\left(\bold{Z}_1U_{(1)}^{1/p}, \ldots, \bold{Z}_{p+1}U_{(p+1)}^{1/p}\right), \left(\bold{Z}_1G_{(1)}^{1/p}, \ldots, \bold{Z}_{p+1}G_{(p+1)}^{1/p} \right)\right)
\end{eqnarray*}
as on $\|\bold{Y}\|\leq K/n^{1/p},$ $\bold{Y}$ has a radially symmetric distribution. Therefore,
\begin{eqnarray*}
&&d_{\rm TV}\left((f(\bold{0})V_p n)^{1/p}(\bold{X}_{(1)}, \ldots, \bold{X}_{(p+1)}), (\bold{Z}_1G_1^{1/p}, \ldots, \bold{Z}_{p+1}G_{p+1}^{1/p})\right)\\
&\leq& o(1) + 2\epsilon + d_{\rm TV} \left(n(U_{(1)}, \ldots, U_{(p+1)}), (G_1, \ldots, G_{p+1})\right)\\
&=& o(1) + 2\epsilon,
\end{eqnarray*}
which concludes the proof.

\begin{theorem}
Assume that $\bold{0}$ is a Lebesgue point of the density $f$ and $f(\bold{0}) > 0$,
and the regression function $r$ is continuously differentiable in a neighborhood of $\bold{0}$.
Then,
\begin{itemize}
\item [(1)]
$(f(\bold{0})V_{p+1} n)^{1/p}\left(\frac{1}{p+1}\sum_{i=1}^{p+1}\sum_{j=1}^{p} r_j^{'}(\bold{0})\bold{X}_{(i, j)} \right)
\xrightarrow[]{\mathcal{D}}\frac{1}{p+1}\sum_{i=1}^{p+1}G_i^{1/p}\left(\sum_{j=1}^p r_{j}^{'}(\bold{0})\bold{Z}_{i, j}\right).$
\item [(2)]
$r_n(\bold{0})-r(\bold{0})=
O_{\mathbb{P}}\left(\frac{1}{\sqrt{p+1}}\left(\frac{p+1}{n}\right)^{1/p}\right).$
\end{itemize}
\end{theorem}
\begin{proof}
By Taylor's series approximation,
\begin{eqnarray*}
r_n(\bold{0})-r(\bold{0})&=& \sum_{i=1}^{p+1}\lambda_{(i)} \left(r(\bold{X}_{(i)})-r(\bold{0})\right)\\
&=& \sum_{i=1}^{p+1} \lambda_{(i)} \left(\sum_{j=1}^{p}r_j^{'}(\bold{0})\bold{X}_{(i, j)}+\psi(\bold{X}_{(i)})\right),
\end{eqnarray*}
(where $\psi(\bold{x})=o(\|\bold{x}\|)$  as $\|\bold{x}\| \rightarrow 0)$. Observe that
$\frac{\sum_{i=1}^{p+1}\psi(\bold{X}_{(i)})}{\|\bold{X}_{(p+1)}\|}\rightarrow0$
in probability,
if $\|\bold{X}_{(p+1)}\|\rightarrow0$ in probability. But, by Lemma \ref{independent_uniform}
$\|\bold{X}_{(p+1)}\|=O_{\mathbb{P}}((\frac{p+1}{n})^{1/p})$, and therefore,
$$\sum_{i=1}^{p+1}\lambda_{(i)}\psi(\bold{X}_{(i)})=o_{\mathbb{P}}\left(\left(\frac{p+1}{n}\right)^{1/p}\right).$$
Still, by Lemma \ref{independent_uniform},
\begin{eqnarray*}
&&(f(\bold{0})V_{p} n)^{1/p}\left(\sum_{i=1}^{p+1}\lambda_{(i)}\sum_{j=1}^{p} r_j^{'}(\bold{0})\bold{X}_{(i, j)} \right)
\xrightarrow[]{\mathcal{D}}\sum_{i=1}^{p+1}\lambda_{(i)}G_i^{1/p}\left(\sum_{j=1}^p r_{j}^{'}(\bold{0})\bold{Z}_{i, j}\right).
\end{eqnarray*}
Define $L_{p+1}=\sum_{i=1}^{p+1}\lambda_iG_i^{1/p}\bold{Z}_{i,1},$ we have
$$\mathbb{E}[L_{p+1}^2|G_1, \ldots, G_{p+1}]=\sum_{i=1}^{p+1}\lambda_i^2G_i^{2/p}\mathbb{E}\bold{Z}_{i, 1}^2=
\mathbb{E}\bold{Z}_{1,1}^2\times \sum_{i=1}^{p+1} \lambda_iG_i^{2/p}.$$
Thus, since $\mathbb{E}\bold{Z}_{1, 1}^2=\frac{\|\bold{Z}_1\|^2}{p}=1/p,$
$$\mathbb{E}L_{p+1}^2\sim \frac{{(p+1)}^{2/p-1}}{p+2},$$
which implies that
\begin{eqnarray*}
r_n(\bold{0})-r(\bold{0})=
O_{\mathbb{P}}\left(\frac{1}{\sqrt{p+1}}\left(\frac{p+1}{n}\right)^{1/p}\right).
\end{eqnarray*}
\end{proof}

\begin{theorem}
Let $\bold{X}_1, \ldots, \bold{X}_n$ be i.i.d samples from a continuous density $f$, which is bounded away from zero and infinity on $[0,1]^p$. Then,
for a function $y=\phi(\bold{x})$ which has an upper-bounded second-order derivative, we have
$$\lim_{n\rightarrow \infty} n^{2/p-1}\mathbb{E}\|\phi-\hat{F}_D\|\leq \frac{3}{8}{{p+1}\choose{2}}^2\|\nabla^2\phi\|\beta\int_{[0,1]^p}f(\bold{x})^{(p-2)/p}d\bold{x}.$$
where $\hat{F}_D$ is a DTL estimate of $\phi$ based on samples $\bold{X}_1, \ldots, \bold{X}_n$, $\|\nabla^2\phi\|$ is the upper bound of the second-order derivative, $\beta$ is a positive constant depending only on the dimension $p$ and density $f$,
and $\|\cdot\|$ is a $L_2$ norm of the continuous functional space.
\end{theorem}
\begin{proof}
%[Proof of Theorem 6]
Define $\mathcal{D}$ as the Delaunay triangulation on $\bold{X}_1, \ldots, \bold{X}_n$, and $\mathcal{S}_1, \ldots, \mathcal{S}_{N_\mathcal{D}}$ as the Delaunay simplices, where $N_\mathcal{D}$ is the number of simplices in $\mathcal{D}$. Define $T_j$ as the average length of the edges of the Delaunay triangle $\mathcal{S}_j$.
By the $L_2$-error bound obtained in the inequality (3.22) in \cite{Waldron1998},
for each simplex $\mathcal{S}_j$, we have
\begin{eqnarray*}
\int_{\bold{x}\in \mathcal{S}_j}\big|\phi(\bold{x})-\hat{F}_D(\bold{x})\big|^2d\bold{x}
\leq \frac{3}{2}h_j^2\|\nabla_j^2\phi\|,
\end{eqnarray*}
where $\|\nabla_j^2\phi\|=\||\sup_{\bold{x}\in\mathcal{S}_j}\nabla^2(\bold{x})\phi|\|$ is the supremum of the second-order derivative of $\phi$ on simplex $\mathcal{S}_j,$ and $h_j$ is the maximum edge of $\mathcal{S}_j$. Then, it can be shown that
\begin{eqnarray*}
\mathbb{E}\|\phi-\hat{F}_D\|&=&\mathbb{E}\sum_{i=1}^n\int_{\bold{x}\in \mathcal{S}_j}\big|\phi(\bold{x})-\hat{\phi}_D(\bold{x})\big|^2d\bold{x}\\
&\leq&\mathbb{E} \sum_{j=1}^{N_\mathcal{D}}\frac{3}{2}h_j^2\|\nabla_j^2\phi\|\\
&\leq&\mathbb{E} \sum_{j=1}^{N_\mathcal{D}}\frac{3}{2}{p+1 \choose 2}^2\left(\frac{T_j}{2}\right)^2\|\nabla_j^2\phi\|\\
&\leq&\mathbb{E} \sum_{j=1}^{N_\mathcal{D}}\frac{3}{2}{p+1 \choose 2}^2\left(\frac{T_j}{2}\right)^2\|\nabla^2\phi\|,
\end{eqnarray*}
where $\|\nabla^2\phi\|=\max_j \|\nabla_j^2\phi\|<\infty$.
Thus,
$$\mathbb{E}\|\phi-\hat{F}_D\|\leq \sum_{j=1}^{N_\mathcal{D}}\frac{3}{8}{{p+1}\choose{2}}^2\|\nabla^2\phi\|\mathbb{E}T_j^2,$$
and by Theorem \ref{theorem:1}, we have
$$\lim_{n\rightarrow \infty} n^{2/p-1}\mathbb{E}\|\phi-\hat{F}_D\|\leq \frac{3}{8}{{p+1}\choose{2}}^2\|\nabla^2\phi\|\beta\int_{[0,1]^p}f(\bold{x})^{(p-2)/p}d\bold{x}.$$
\end{proof}

\begin{theorem}[Local Adaptivity]
Let $\bold{X}_1, \ldots, \bold{X}_n$ be i.i.d samples from a continuous density $f$, which is bounded away from zero and infinity on $[0,1]^p$.
Assume that the regression function $r$ is continuously differentiable in $[0,1]^p$.
Then, for any point $\bold{x}\in[0,1]^p$, $\hat{F}_D'(\bold{x})\rightarrow r'(\bold{x})$, as $n\rightarrow\infty,$ where $\hat{F}_D'(\bold{x})$ is the gradient of the DTL at point $\bold{x}.$
\end{theorem}

\begin{proof}
For any fixed point $\bold{x}$, by Taylor series approximation,
\begin{eqnarray*}
r_n(\bold{x})-r(\bold{x})&=&\frac{1}{p+1} \sum_{i=1}^{p+1} \left(r(\bold{X}_{(i)})-r(\bold{x})\right)\\
&=&\frac{1}{p+1} \sum_{i=1}^{p+1} \left(\sum_{j=1}^{p}r_j^{'}(\bold{x})\bold{X}_{(i, j)}+\psi(\bold{X}_{(i)})\right),
\end{eqnarray*}
(where $\psi(\bold{x})=o(\|\bold{x}\|)$  as $\|\bold{x}\| \rightarrow 0)$. Observe that
$\frac{\sum_{i=1}^{p+1}\psi(\bold{X}_{(i)})}{\|\bold{X}_{(p+1)}\|}\rightarrow0$
in probability,
if $\|\bold{X}_{(p+1)}\|\rightarrow0$ in probability.
Based on Theorem \ref{theorem:1} and Corollary \ref{corollary:1}, we have
$\bold{X}_{(i)}\rightarrow \bold{X}$, $i=1, \ldots, p+1.$
\end{proof}
Based on the multivariate mean-value theorem,
there exists a point $\xi\in\mathcal{S}(\bold{x})$ such that
$\hat{F}_D'(\bold{x})= r'(\xi).$ As $n\rightarrow \infty$, $\lim_{n\rightarrow\infty}\hat{F}_D'(\bold{x})=\lim_{n\rightarrow\infty}r'(\xi_n)=r'(\lim_{n\rightarrow\infty}\xi_n)=r'(\bold{x}).$

\subsection{Regularized Estimates}
Under regularizations, the estimated parameter vector is given by
$$\hat{\bold{\Psi}}_{\lambda}= \arginf_\bold{\Psi} L(\mathcal{Y}, F_{D}(\mathcal{X}; \bold{\Psi})) + \lambda R(\bold{\Psi}).$$
To see how accurate the regularized estimate is, we compare $\hat{\bold{\Psi}}_{\lambda}$ with the population optimal solution
$\bold{\Psi}^* = \arginf_\bold{\Psi} \mathbb{E}_{\mathcal{P_{\mathcal{X}, \mathcal{Y}}}} L(\mathcal{Y}, F_{D}(\mathcal{X}; \bold{\Psi})).$
\begin{definition}
Given the location parameter vector $\bold{\Psi} \in R^n$, the regularization function $R$ is said to be decomposable with respect to $(\mathcal{M}, \bar{\mathcal{M}}^{\perp})$, if a pair of subspaces $\mathcal{M} \subseteq \bar{\mathcal{M}}$ of $R^n$,
$R(\bold{u} + \bold{v}) = R(\bold{u}) + R(\bold{v})$
for all $\bold{u} \in \mathcal{M}$ and $\bold{v}\in \bar{\mathcal{M}}^{\perp}$,
where $\bar{\mathcal{M}}^{\perp} = \{\bold{v} \in R^n | \langle \bold{u}, \bold{v}\rangle = 0, \forall \bold{u} \in \bar{\mathcal{M}}\}.$
\end{definition}
\begin{definition}
The loss function satisfies a restricted
strong convexity condition (RSCC) with curvature
$\kappa_L > 0$,  if the optimal solution $\bold{\Psi}^*\in \mathcal{\mathcal{M}}$,
$\delta L(\Delta, \bold{\Psi}^*)\geq \kappa_{L}\|\Delta\|^2$
for all $\Delta \in \mathbb{C}(\mathcal{\mathcal{M}}, \bar{\mathcal{\mathcal{M}}}^{\perp};\bold{\Psi}^*)$, where
$ \mathbb{C}(\mathcal{M}, \bar{\mathcal{M}}^{\perp};\bold{\Psi}^*)=\{ \Delta \in R^n | R(\Delta_{\bar{\mathcal{M}}^{\perp}})\leq 3R(\Delta_{\bar{\mathcal{M}}} ) + 4R(\bold{\Psi}_{\mathcal{M}^{\perp}}^*)\}$.

\end{definition}
\begin{theorem}
Based on a strictly positive regularization constant
$\lambda \geq 2R(\nabla L(\bold{\Psi}^*))$,
if the loss function satisfies RSCC,
then we have $\|\hat{\bold{\Psi}}_{\lambda}-\bold{\Psi}^*\|\leq  9\frac{\lambda^2}{\kappa_L}\bold{\Phi}^2(\bar{\mathcal{M}})$ and $R(\hat{\bold{\Psi}}_{\lambda}-\bold{\Psi}^*)\leq  12\frac{\lambda}{\kappa_L}\bold{\Phi}^2(\bar{\mathcal{M}}),$ where $\bold{\Phi}(\bar{\mathcal{M}})=\sup_{\bold{u}\in \mathcal{M}\{\bold{0}\}}\frac{R(\bold{u})}{\|\bold{u}\|}.$
\end{theorem}
\begin{proof}
We first show that the regularization function of DTL is decomposable, and then we use Theorem 1 and Corollary 1 in \cite{Negahban2012} to prove the theorem.
For a given set of parameters $\mathcal{M}$, define $\mathcal{H}(\mathcal{M})$ as its convex hull, and $\bar{\mathcal{M}}$ as the set of points that are connected to any point of $\mathcal{H}(\mathcal{M})$ by an edge of the Delaunay triangulation $\mathcal{D}(\mathcal{X})$. Then, it is clear that based on the definition of geometric regularization function, we have $R(\bold{u} + \bold{v}) = R(\bold{u}) + R(\bold{v})$ for all $\bold{u} \in \mathcal{M}$ and $\bold{v} \in \bar{\mathcal{M}}^{\perp}$.
Thus if the loss function satisfies RSCC for such an $\mathcal{M}$, then we can make the conclusion.

\end{proof}

\section{Comparisons of Statistical Learners}\label{Comparisons}
\begin{figure}[t!]
    \subfloat[DTL]{%
    \includegraphics[width=0.3\linewidth]{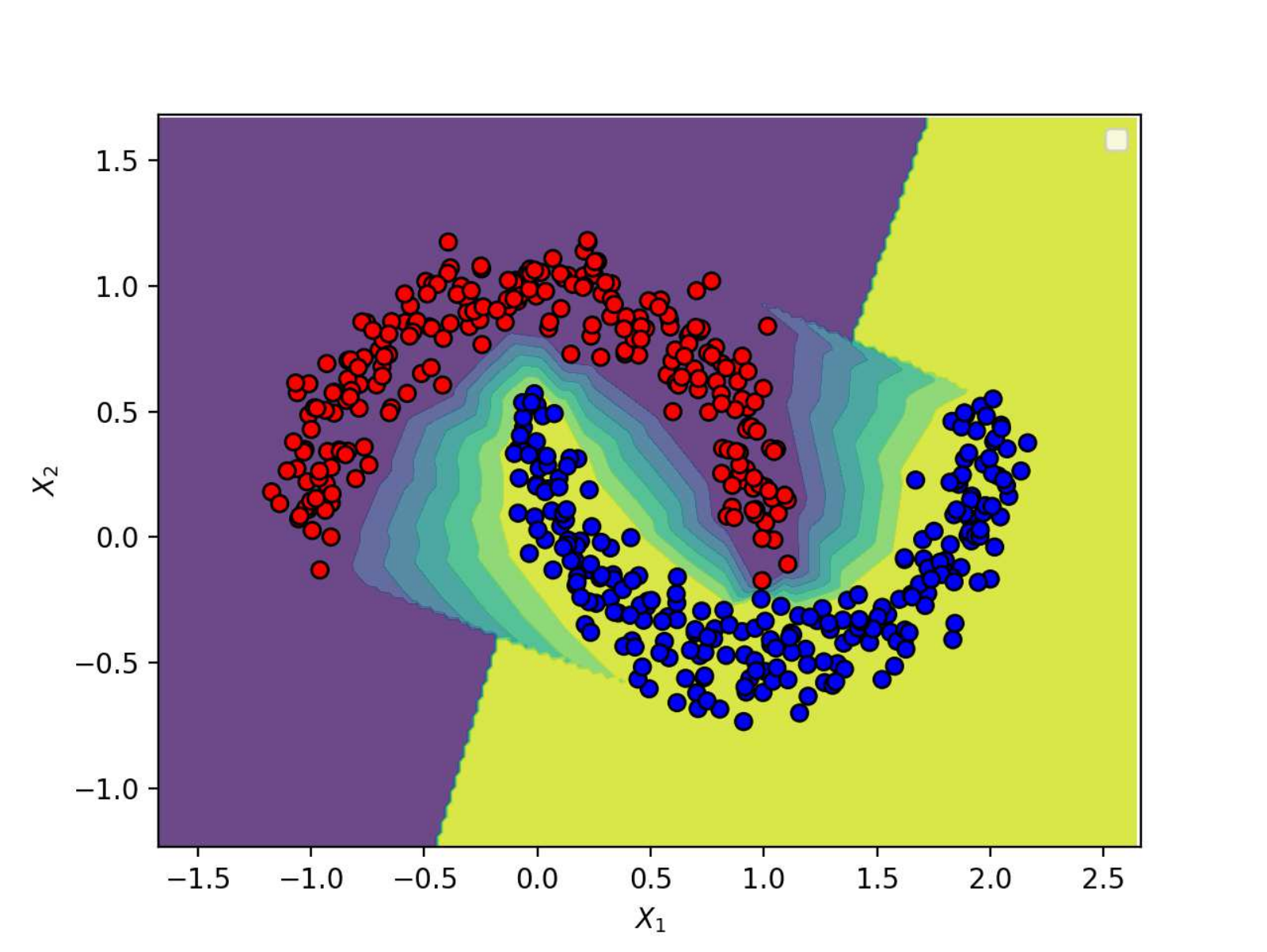}
    }
    \hfill
    \subfloat[Neural Network]{%
    \includegraphics[width=0.3\linewidth]{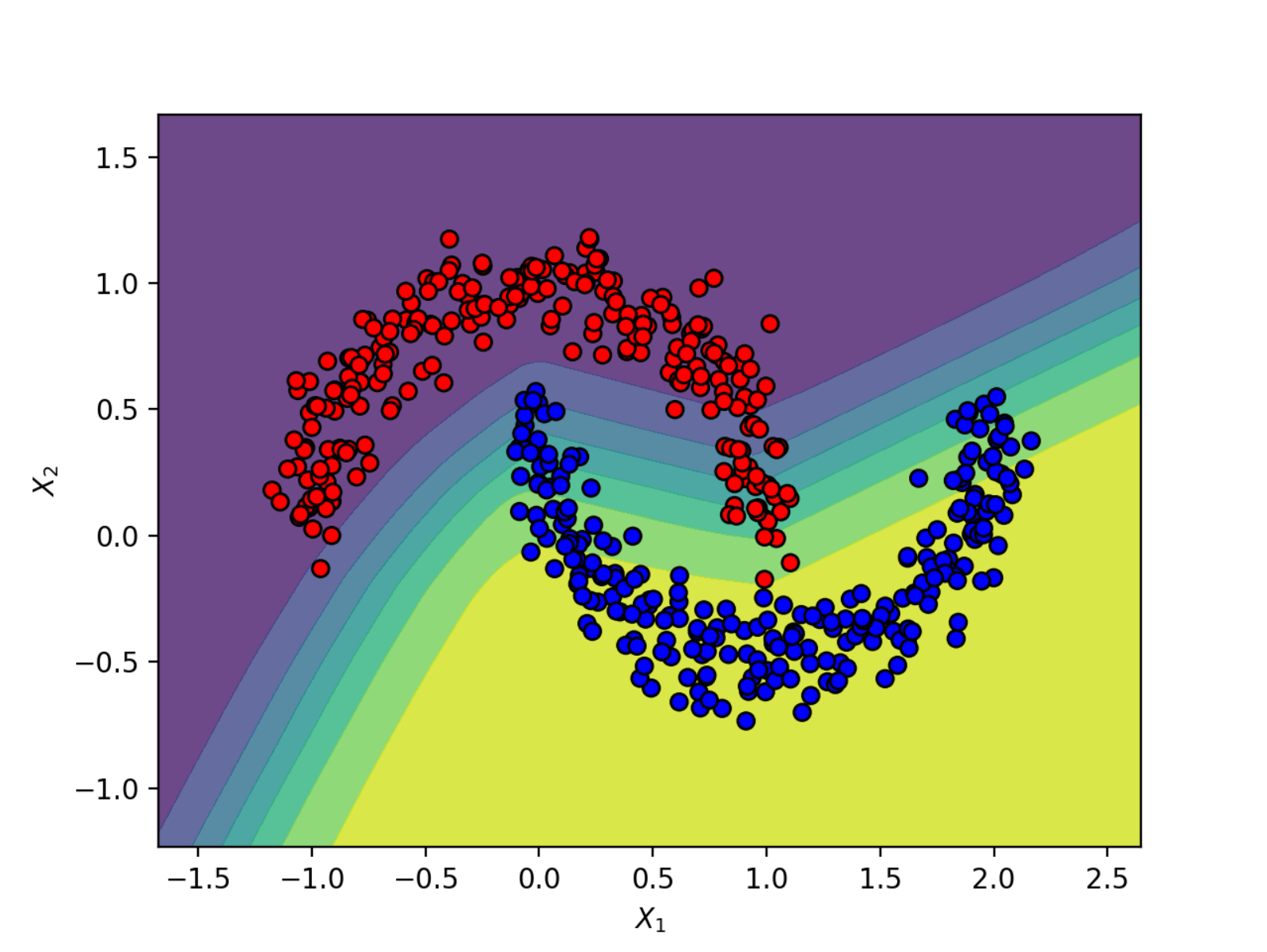}
    }
    \hfill
    \subfloat[Decision tree]{%
    \includegraphics[width=0.3\linewidth]{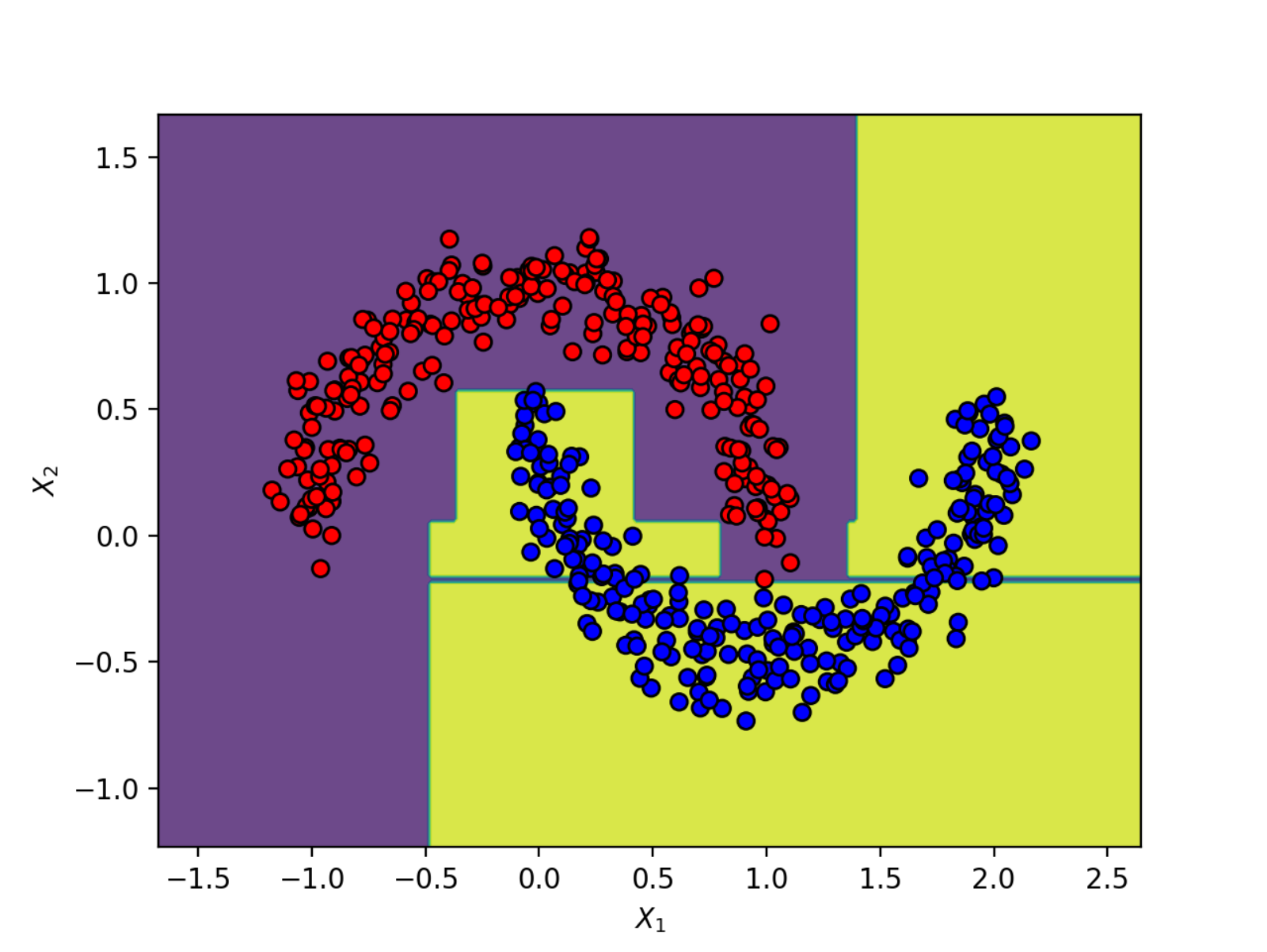}
    }
    \hfill
    \caption{Classification results of the `make moons' dataset}
    \label{fig:classification_comparison2}
    \end{figure}

\begin{figure}[t!]
    \subfloat[DTL]{%
    \includegraphics[width=0.3\linewidth]{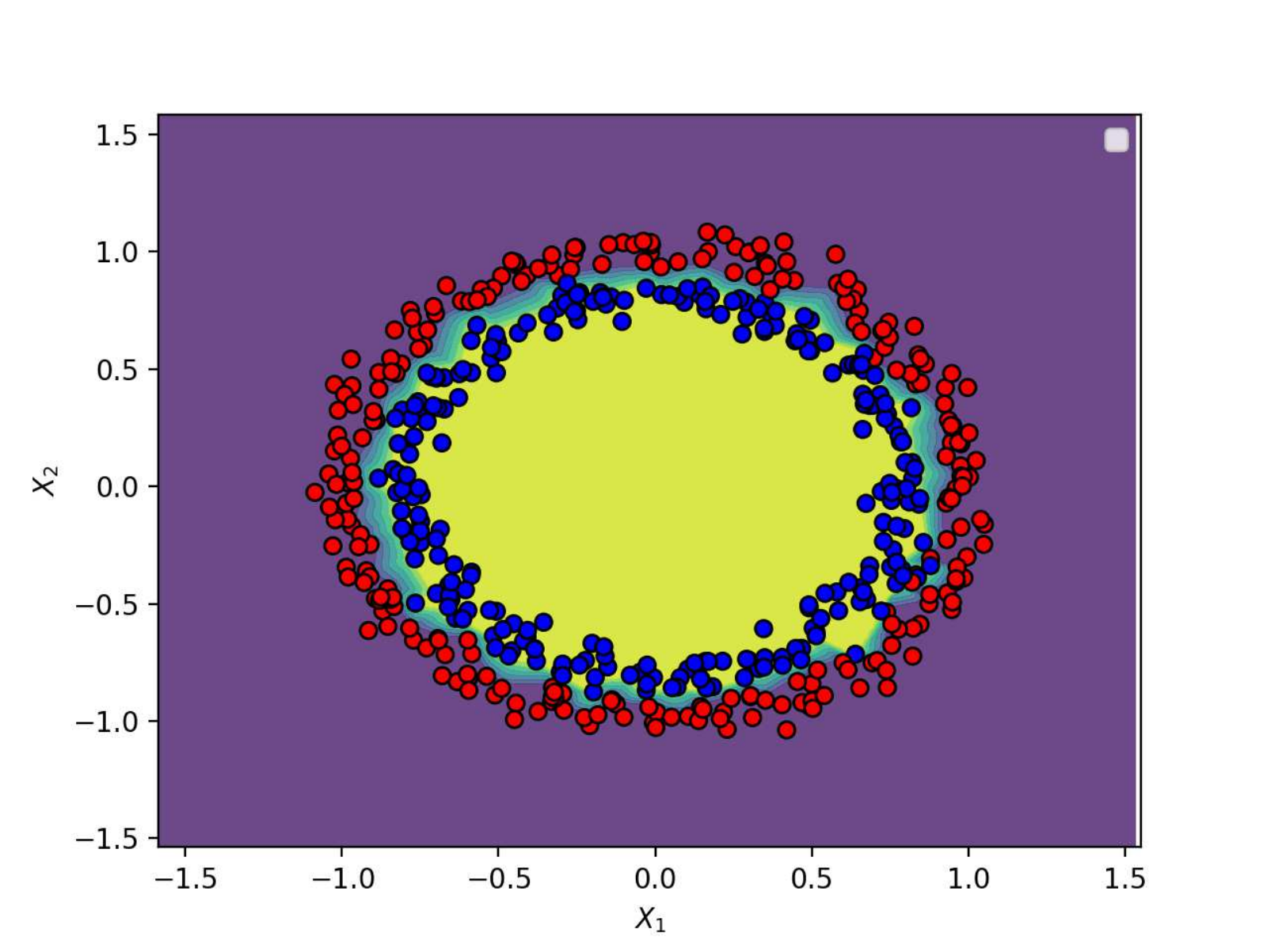}
    }
    \hfill
    \subfloat[Neural Network]{%
    \includegraphics[width=0.3\linewidth]{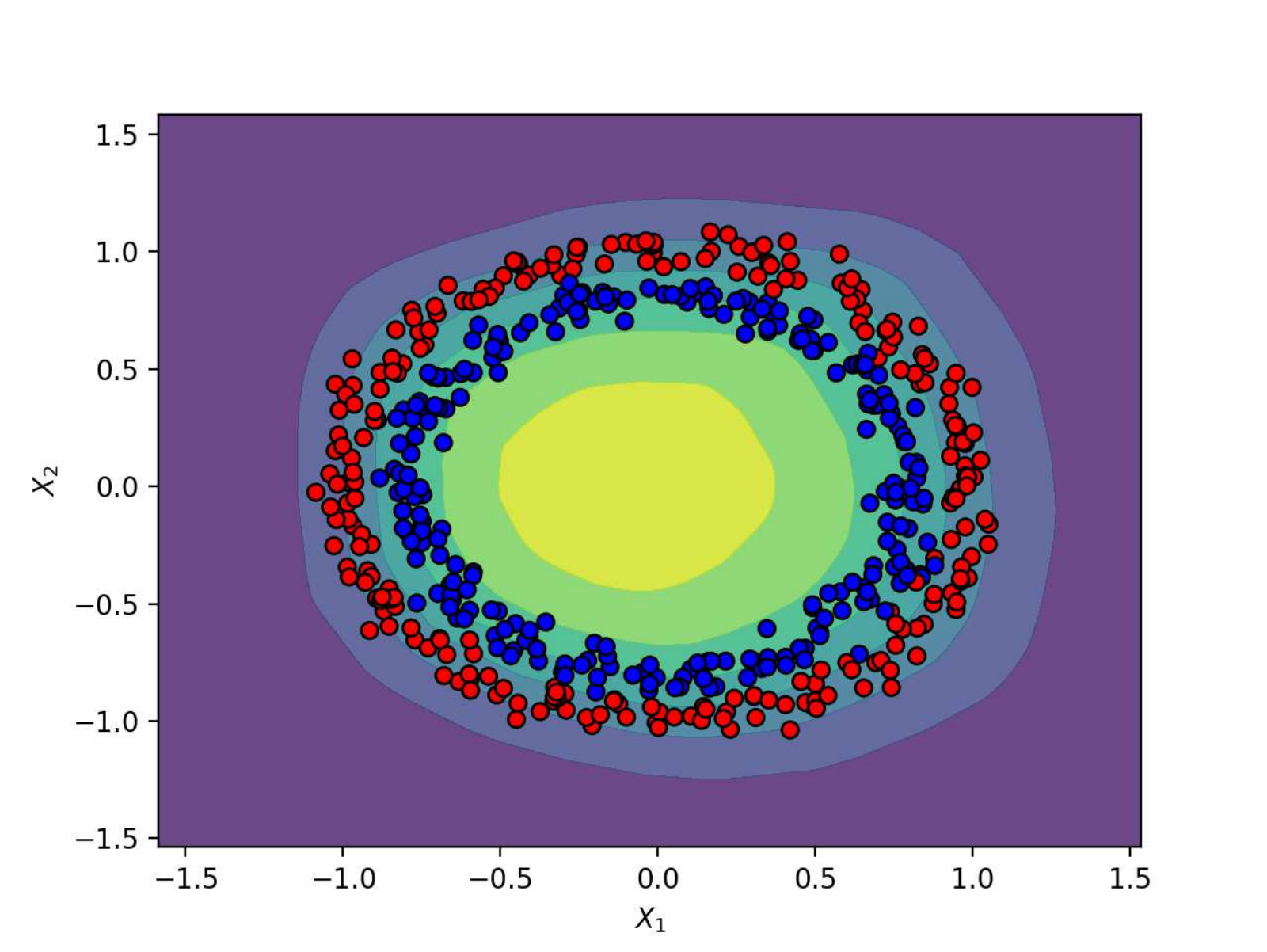}
    }
    \hfill
    \subfloat[Decision tree]{%
    \includegraphics[width=0.3\linewidth]{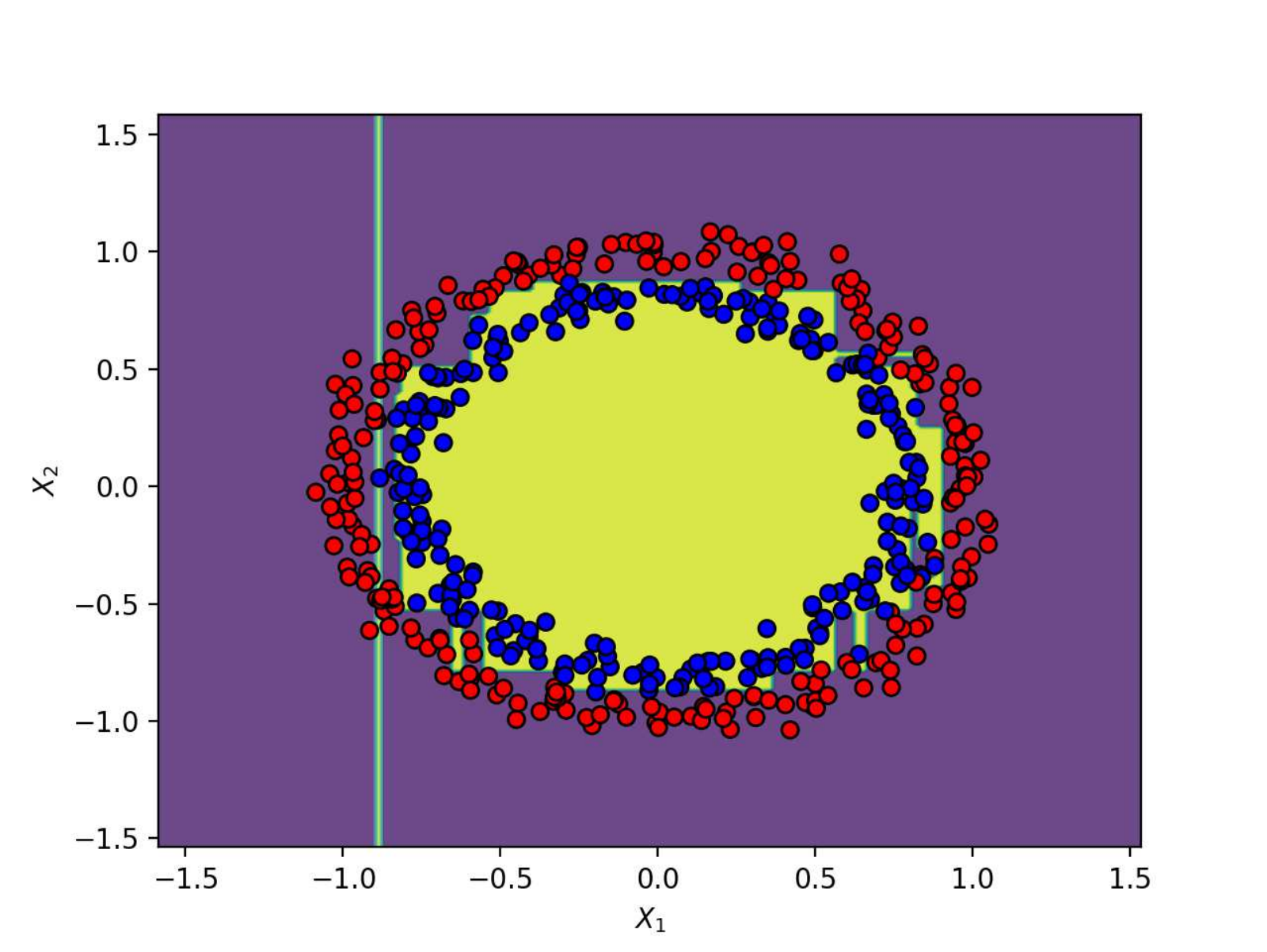}
    }

    \caption{Classification results of the `circles' dataset}
\label{fig:classification_comparison3}
    \end{figure}
We use two artificial classification datasets to visualize how the DTL-based methods
differ from the neural network and tree-based methods. In particular, we demonstrate the smoothness and robustness of the DTL-based methods, when handling feature interaction problems.

The artificial datasets are generated from three models of the Python Scikit learn package.
The `moons' model has two clusters of points with moon shapes, and
the `circles' model consists of two clusters of points distributed
along two circles with different radiuses.
As shown in Figures \ref{fig:classification_comparison2}--\ref{fig:classification_comparison3},
the data are displayed with blue points for $y=1$ and red points for $y=0$.
Three classifiers are considered including the DTL, neural network,
decision tree,
and the estimated probability function of each classifier is plotted with color maps.

We first illustrate the local adaptivity of the DTL. As shown in
panel (a) of Figures \ref{fig:classification_comparison2}--\ref{fig:classification_comparison3},
the estimated probability functions of the DTL are piecewise linear
in the convex hull of the observations. This feature provides the DTL the
advantage when approximating smooth boundary, since it can locally choose to use a piecewise
linear model to estimate the probability function.
Figure \ref{fig:classification_comparison2} (b) exhibits a less adaptive
classification probability (either 0 or 1) using neural network, and
Figure \ref{fig:classification_comparison2} (c) shows a clear stairwise shape
of the classification region using the decision tree.

We next point out that the DTL can more
reliably capture the feature interactions.
As exhibited in Figure \ref{fig:classification_comparison2} (c) and Figure \ref{fig:classification_comparison3} (c),
there are
some spiky regions in the estimated probability functions
using the decision tree. It shows the weakness of the decision tree
in capturing the feature interactions via the marginal tree splitting approach.

\section{Discussion}\label{Discussion}
We have presented the DTL as a differentiable and nonparametric algorithm that has simple geometric interpretations.
The geometrical and statistical properties of the DTL are investigated
to explain its advantages under various settings,
including general regression, classification models, smooth and noiseless models.
These properties signify the importance of developing more applications
based on this piecewise linear learner, e.g., the DTL can be used as an alternative
approach to substituting the nonlinear activation function in the neural network.
The DTL can also be generalized into manifold learning approaches,
with multi-dimensional input and output.
There exist  a series of open questions warranting further investigations.
In terms of theory, we have focused on low-dimensional settings ($p<n$)
for simplicity, while it would be more interesting to generalize the DTL
into high-dimensional settings. As for computation, investigation on
GPU-based parallel computing for DTL is also warranted.

\end{document}